\documentclass[a4paper,11pt]{article}
\usepackage{bm}
\usepackage[utf8]{inputenc}
\usepackage[T1]{fontenc}
\usepackage[english]{babel}
\usepackage{amsmath,amssymb,amsfonts,siunitx,commath}
\usepackage{amsthm}
\usepackage{tikz,url}
\usepackage{geometry}
\usepackage{mathtools}
\mathtoolsset{showonlyrefs}
\usepackage{subcaption}
\usepackage{algorithm}
\usepackage{algpseudocode}
\usepackage{hyperref}
\usepackage{cancel}
\usepackage{enumerate}
\usepackage{listings}

\newcommand{\zb}[1]{\ensuremath{\boldsymbol{#1}}}
\newcommand{\tT}{\mathrm{T}}
\newcommand{\X}{\mathbb{X}}
\newcommand{\Y}{\mathbb{Y}}
\newcommand{\R}{\mathbb{R}}
\newcommand{\Z}{\mathbb{Z}}
\newcommand{\N}{\mathbb{N}}
\newcommand{\PP}{\mathbb{P}}

\newcommand{\diff}{\textnormal{d}}
\newcommand{\assign}{\coloneqq}
\newcommand{\const}{\text{const}}
\DeclareMathOperator{\KL}{KL}
\DeclareMathOperator{\spt}{spt}
\DeclareMathOperator*{\argmin}{arg\,min}

\newcommand{\rightweaks}{\stackrel{\ast}{\rightarrow}}

\newtheorem{theorem}{Theorem}
\newtheorem{assumption}[theorem]{Assumption}
\newtheorem{lemma}[theorem]{Lemma} 
\newtheorem{proposition}[theorem]{Proposition}
\newtheorem{remark}[theorem]{Remark}
\newtheorem{definition}[theorem]{Definition}

\newtheorem{corollary}[theorem]{Corollary}
\theoremstyle{definition}

\begin{document}

\author{
Florian Beier\footnote{Institute of Mathematics,
 	Technische Universit\"at Berlin,
 	Strasse des 17. Juni 136, 10587 Berlin, Germany,
 	f.beier@tu-berlin.de, steidl@math.tu-berlin.de}
	\and
Hancheng Bi\footnote{Institute of Computer Science,
  Georg-August-Universit\"at Göttingen,
  Goldschmidtstraße 7, 37077 Göttingen, Germany,
  \{bi, clement.sarrazin, schmitzer\}@cs.uni-goettingen.de
}
	\and
Clément Sarrazin\footnotemark[2]
	\and
Bernhard Schmitzer\footnotemark[2]
    \and
Gabriele Steidl\footnotemark[1]
	}
\date{}
\title{Transfer Operators from Batches of Unpaired Points 
via Entropic Transport Kernels}

\maketitle

\begin{abstract}
In this paper, we are concerned with estimating the joint probability
of random variables $X$ and $Y$,  given $N$ independent observation blocks
$(\zb x^i,\zb y^i)$, $i=1,\ldots,N$, 
each of $M$ samples 
$(\zb x^i,\zb y^i) = \big( (x^i_j, y^i_{\sigma^i(j)}) \big)_{j=1}^M$,
where $\sigma^i$ denotes an unknown permutation 
of i.i.d.~sampled pairs
$(x^i_j,y_j^i)$, $j=1,\ldots,M$. 
This means that the internal ordering 
of the $M$ samples within an observation block is not known. 
We derive a maximum-likelihood inference functional, propose a computationally tractable approximation and analyze their properties. In particular, we prove a $\Gamma$-convergence
result showing that we can recover the true density 
from empirical approximations as the number $N$ of blocks goes to infinity.
Using entropic optimal transport kernels, we model a class of  hypothesis spaces of density functions over which the inference functional can be minimized.
This hypothesis class is particularly suited for
approximate  inference of transfer operators from data. 
We solve the resulting discrete minimization problem by 
a modification of the EMML algorithm to take addional transition probability
constraints into account and prove the convergence of this algorithm.
Proof-of-concept examples demonstrate the potential of our method.
\end{abstract}

%------------------------------------------------------------------
\section{Introduction}\label{sec:Intro}
%------------------------------------------------------------------
Given independent and identically distributed (i.i.d.)\ realizations $(x^i,y^i)_{i=1}^N$ from the joint distribution $\pi$ of two random variables $X$ and $Y$ with values in $\X$ and $\Y$, a relevant problem is to find a good approximation of $\pi$ for further analysis.
Assuming that $\pi$ has density $p$ with respect to some product measure $\mu \otimes \nu$, i.e., $\pi = p \cdot (\mu \otimes \nu)$, a common approach to approximate $p$ is to minimize    
the negative log-likelihood based on the observations 
\begin{align} \label{eq:IntroML}
\argmin_{q \in Q} - \frac1N \sum_{i=1}^N \log \left(q(x^i,y^i) \right).
\end{align}
over a parametric or non-parametric hypothesis space of densities
$Q \subset \mathcal C(\X \times \Y)$.

\paragraph{Analysis of dynamical systems.}
An example for such a problem is when 
$X$ and $Y$ 
represent two successive time-steps $t$ and $t+1$ of a time-discrete dynamical system.
In this case, let $\mu$ be the marginal distribution of $X$ and $\nu$ be the marginal distribution of $Y$. Then $p(x,\cdot) \cdot \nu$ is the distribution of $Y$ conditioned on $X=x$.
The family $(p(x,\cdot))_x$ are then called the transition probability densities.
Consider a cohort of particles that at time $t$ is distributed according to
$u_t \cdot \mu \in \mathcal P(\X)$ for some probability density 
$u_t \in L^1(\mu)$ 
and introduce the operator 
$\mathcal T : L^1(\mu) \to L^1(\nu)$ by 
$\mathcal Tu_t \assign \int_{\X} p(x,\cdot) u_t(x)\,\diff \mu(x)$. 
Then at time $t+1$ the distribution of the particle cohort has the density $u_{t+1}\assign \mathcal Tu_t$ with respect to $\nu$.
The map $\mathcal T$ is called the transfer operator of the system. 
If $\mathcal T$ is a compact operator from $L^2(\mu)$ into $L^2(\nu)$, then the singular value decomposition of $\mathcal T$ can reveal macroscopic features of the system, such as coherent sets \cite{Froyland2010Coherent,FROYLAND20131,froyland2014almost}. 
If $\X=\Y$ and $\mu=\nu$, eigendecomposition of $\mathcal T$ can be employed in a similar way, for instance to detect almost-stationary sets or  macroscopic approximate cyclic behaviour \cite{DelJun1999,KoltaiWeiss2020}.
Estimating the conditional transition probabilities $(p(x,\cdot) \cdot \nu)_{x \in \X}$ or the operator $\mathcal T$ (in such a way that its singular value or eigendecomposition can be approximated) from sample data $(x^i,y^i)_{i=1}^N$ is a relevant problem.
The literature on the analysis of dynamical systems via spectral analysis of the transfer operator and the approximation of the latter from data is vast and we refer to \cite{DelJun1999,Koopmanism2012,FROYLAND20131,KluKolSch2016,TransferOperatorApproxReview18} and references therein or the monograph \cite{ErgodicOperators} as exemplary starting points. As an example for more recent developments see, for instance, \cite{BittracherRC2021}.
To get meaningful results in such an application, \eqref{eq:IntroML} should be minimized over a suitable class of hypothesis densities $Q$ such that for all $q \in Q$, 
the measure $q(x,\cdot) \cdot \nu$ is indeed a conditional probability distribution for $\mu$-almost all $x \in X$, i.e.~they are transition probability densities.

In this paper, we consider the following variant of the above estimation problem: 
\begin{assumption}[Sample generation]\label{asp:main}
We assume that we observe $N$ batches of $M$ i.i.d.~pairs of samples from the joint distribution of random variables
$(X,Y)$, but their actual pairings within each batch remain unknown. 
More precisely,
for each $i \in \{1,\ldots,N\}$, the $M$ pairs $((x_{j}^i,y_{j}^i))_{j=1}^M$ are sampled i.i.d.~from the joint law of $(X,Y)$, and in addition, 
a permutation $\sigma^i$ is sampled uniformly 
from the permutation group of $\{1,\ldots,M\}$, 
so that we have only access to
$((x_{j}^i,y^i_{\sigma^i(j)}))_{j=1}^M$.
\end{assumption}

This models, for instance, the experiment where we observe a group of $M$ particles evolving according to a dynamical system, but the particles are indistinguishable and their identity cannot be tracked between time steps. The extreme case that
we observe only a single batch, $N=1$, 
was considered in connection with transport phenomena and two time steps for the Schr\"odinger problem 
in \cite{OTCoherentSet2021} and for Gromov--Wasserstein transport problems in \cite{beier_gwt}. 
Similarly, the approach in \cite{Lavenant2021} for tracking cell differentiation can be interpreted as the setting $N=1$ and $M\to \infty$ and in both cases entropy regularized optimal transport is used as a prior to solve this vastly underdetermined problem. In \cite{DalitzSSVM2017} a variant of the problem is considered, where only tomographic projections of the positions are observed. Leveraging optimal transport as a prior for displacement and compressed sensing theory for sparse reconstruction it is shown that the particle positions and their associations can be recovered correctly with high probability for $N=1$ and $M$ sufficiently small (but greater 1).

\paragraph{Particle colocalization.}
The above estimation problem might also serve as a simplified model for the analysis of particle colocalization in super-resolution microscopy.
In this case, $\X=\Y$ is the image domain and $X$ and $Y$ describe the random locations of two species of fluorescent markers.
Let us assume for simplicity that in each image we observe $M$ markers of each species, that correspond to $M$ i.i.d.~sampled pairs from an unknown $\pi$, with their pairwise association remaining unobserved. If the particles form tightly bound pairs, $\pi$ would have most of its mass concentrated near the diagonal on $\X \times \Y$, such that with high probability $X \approx Y$. If the particle positions are practically independent, $\pi$ would be approximately equal to the product measure $\mu \otimes \nu$, i.e.~$p(x,y) \approx 1$ for all $(x,y)$.
Therefore, inferring $\pi$ from such data can reveal information about the interaction between the two species.
Of course, the real problem is more complicated, due to effects such as incomplete labeling efficiency (i.e.~not all markers are actually visible in the images), incomplete pairing (not all markers are necessarily paired with one of the other species, even if their positions were highly dependent, for instance due to different global abundances) and related issues. At this point were merely consider this as a conceptual study.
We refer to \cite{TamelingColoc2021} and references therein for an exposition of the colocalization problem in super-resolution microscopy and for an analysis method based on optimal transport.

\paragraph{Outline of the paper.}
We start by providing the necessary notation in Section \ref{sec:notation}.
In Section \ref{sec:Functionals}, we develop a functional for solving our inference problem.
We first derive a maximum likelihood estimator in the spirit of \eqref{eq:IntroML}, but under Assumption \ref{asp:main} for sampling. This turns out to become numerically intractable very quickly as $M$ increases due to the vast number of potential permutations. We therefore give a tractable approximate negative log-likelihood and establish its basic properties. In particular, we prove a $\Gamma$-convergence type result that shows that we can recover the true density $p$ as $N \to \infty$.
In Section \ref{sec:Kernels}, we propose a class of nonparametric hypothesis spaces $Q$. For this, we use the concept of entropic optimal transport, where the regularization parameter $\varepsilon$ approximately plays the role of the squared kernel bandwidth, thus providing a means to control the complexity or bias of $Q$.
This construction is robust with respect to approximating the marginal distributions $\mu$ and $\nu$, e.g.~by sampling and discretization. We give a corresponding $\Gamma$-convergence result.
In addition, entropic optimal transport provides a simple way to implement the constraint that $(q(x,\cdot))_x$ is a family of transition probability densities. The introduced hypothesis class is therefore particularly suited for the approximate inference of transfer operators from data.
 Our use of entropic transport for the estimation of smoothed transfer operators from data is an extension of the method in \cite{EntropicTransfer22} (which deals with the case $M=1$ for deterministic systems) and we briefly discuss related results on spectral convergence.
Solving the resulting discrete minimization problem is the content of Section \ref{sec:Algorithm}. We will see that the unconstrained problem can be solved by the well-known EMML algorithm.
For the problem with added transition probability constraint, we modify the algorithm while preserving its monotone convergence and give a convergence proof in Appendix \ref{subsec:cEMML}.
Finally, in Section \ref{sec:numerics}, we illustrate the performance of our method by several numerical examples on dynamic system analysis, illustrating the role of the parameters $M$, $N$ and $\varepsilon$. In particular, we empirically estimate the stability of our inference method for large $M$ and find that it can still extract information from the samples in this regime.
Conclusions and a list of open questions are given in Section \ref{sec:Conclusion}.

%------------------------------------------------------------------
\section{Notation} \label{sec:notation}
%------------------------------------------------------------------
Let $\X$ be a compact metric space.
By $\mathcal M(\X)$ we denote the space of the Borel  measures on  $\mathbb X$, 
by $\mathcal M_+(\X)$ the subset of non-negative measures
and by $\mathcal P(\X)$ the probability measures.
For a measure
$\mu \in \mathcal{M}_+(\X)$, 
its support is defined by
\[
\spt \mu \coloneqq 
\{x \in \X: 
\text{ for all open neighbourhoods } N_x 
\text{ of } x \text{ it holds } \mu(N_x) > 0
\}.
\]
We identify $\mathcal M (\X)$ with the topological dual of the Banach space of continuous functions $\mathcal C(\X)$ equipped with the supremum norm $\| \cdot \|_\infty$. Let $\mathcal C_+(\X)$ be the
non-negative continuous functions on $\X$.
By $\mu^{\otimes M}$, we denote the $M$-fold product measure of $\mu$ on $\mathbb X^M$.
Then, for $\zb x \coloneqq (x_1,\ldots,x_M)$, 
we set
$\diff \mu^{\otimes M} (\zb x) \coloneqq \diff \mu(x_1) \ldots \diff \mu (x_M)$.
Furthermore, for functions $p: \X \to \R$, 
let $p^{\otimes M}: \X^M \to \R$ be given by 
$p^{\otimes M}(x_1,\ldots,x_m) \coloneqq \prod_{j=1}^M p(x_j)$.

For two spaces $\X$, $\Y$ and a measurable map $T : \X \to \Y$, 
the push-forward measure of $\mu \in \mathcal M(\X)$ under $T$, denoted by $T_\# \mu \in \mathcal M(\Y)$, is defined by  
$T_\# \mu := \mu \circ T^{-1}$.
Note that a random variable $X: (\Omega,\Sigma,\mathbb P) \to   \X$ has law $P_X = X_\# \mathbb P$.
For a set $Q$, let $\iota_Q$ be the indicator function of $Q$, i.e., $\iota_Q(x) = 0$ if $x \in Q$ and 
$\iota_Q(x) = +\infty$ otherwise.

A measure $\mu \in \mathcal M(\X)$ 
is called absolutely continuous 
with respect to $\tilde{\mu} \in \mathcal{M}(\X)$, 
and we write $\mu \ll \tilde{\mu}$, 
if for every Borel-measurable $A \subset \X$ 
with $\tilde{\mu}(A) = 0$ it holds $\mu(A) = 0$.
By $\KL: \mathcal M_+(\X) \times \mathcal M_+(\X) \to \R_+ \cup \{\infty\}$ we denote the Kullback--Leibler divergence
defined for $\mu, \tilde \mu \in \mathcal M_+(\X)$ by
$$\KL(\mu|\tilde \mu) 
\coloneqq 
\begin{cases}
	\int_{\X} \varphi\left(\frac{{\rm d} \mu}{{\rm d} \tilde \mu}\right)\,\text{d} \tilde \mu
	& \text{if } \mu \ll \tilde \mu, \\
	+\infty & \text{otherwise,}
	\end{cases}
$$
with 
$\varphi(s) \coloneqq s \log s - s + 1$ for  $s>0$ 
and 
$\varphi(0) \coloneqq 1$.
It holds that 
\begin{equation}\label{eq:Gibbs}
\KL(\mu|\tilde \mu) = 0 \quad \text{ if and only if }\quad \mu = \tilde \mu.
\end{equation}
In particular, we have for $\mu = p \cdot \lambda \in \mathcal P(\X)$ 
and 
$\tilde \mu = \tilde p \cdot \lambda\in \mathcal P(\X)$ with 
$\lambda \in \mathcal M_+(\X)$
that
\begin{equation} \label{kl}
\KL(\mu|\tilde \mu) = \int_{\X} p \log \Big(\frac{p}{\tilde p} \Big) \, \text{d} \lambda
= \int_{\X} \log (p) \, \text{d} \mu - \int_{\X} \log (\tilde p) \, \text{d} \mu.
\end{equation}
In our algorithmic part, we will also deal with the KL divergence of vectors 
$\zb p, \tilde{\zb p} \in \R^N$ in the probability simplex
$$
\triangle_N \assign \big\{\zb p \in  \R^N_{\ge 0}: \sum_{i=1}^N p_i = 1 \big\}.
$$
Then the KL divergence between $\zb p$ and $\tilde {\zb p}$ is simply given by
\begin{equation} \label{kl_discrete}
\KL(\zb p|\tilde{\zb p}) \assign \sum_{i=1}^N  p_i \log \Big(\frac{p_i}{\tilde p_i} \Big) - \sum_{i=1}^N  p_i + \sum_{i=1}^N  \tilde p_i,
\end{equation}
where $\KL(\zb p|\tilde{\zb p}) \assign +\infty$ if $\tilde p_i = 0$ and $p_i \not = 0$
for some $i \in [N]$ and $0 \log 0 \assign 0$.
We write $[N]$ for the set $\{1,\ldots,N\}$.

%------------------------------------------------------------------
\section{Inference functionals}\label{sec:Functionals}
%-----------------------------------------------------------------
In this section we construct functionals for inferring the true density $p$ of $\pi=p \cdot \mu \otimes \nu$ from observations and establish their basic properties.
We start by considering the random variable
behind our sample generation in Assumption \ref{asp:main} in Subsection \ref{subsec:31} and formulate the corresponding log-likelihood functional, referred to as \emph{permutation functional}.
As $M$ increases, computationally this functional quickly becomes intractable due to the large number of potential permutations.
Therefore, in Subsection \ref{subsec:32}, we  propose an \emph{approximate inference functional}. Finally, in Subsection \ref{sec:FunctionalsBasic}, we establish basic continuity and $\Gamma$-convergence results of the approximate functional. Furthermore, we show that, in the limit $N \to \infty$, its global minimizer will be $p$, and discuss how the approximation relates to the permutation functional.

%---------------------------------------------
\subsection{Modeling and permutation functional} \label{subsec:31}
%---------------------------------------------

Let $\mathbb X, \mathbb Y$ be compact metric spaces
and
$X: (\Omega,\Sigma,\mathbb P) \to   \X$, $Y:(\Omega,\Sigma,\mathbb P) \to   \Y$
random variables with joint law $P_{(X,Y)} = \pi \in \mathcal P(\X \times \Y)$.
The marginal laws of $X$ and $Y$ are then given by 
the push forward of the corresponding projections
$\mu \assign (P_1)_\# \pi$ and $\nu \assign (P_2)_\# \pi$, respectively.
We suppose that there exists $p\in C(\X \times \Y)$ 
such that $\pi = p \cdot (\mu \otimes \nu)$.
Note that this implies 
\begin{align}
	\int_{\X} p(x',y)\,\diff \mu(x')  =1 \quad \text{and} \quad
	\int_{\Y} p(x,y')\,\diff \nu(y')  =1
\label{eq:pMarginals}
\end{align}
for all $(x,y) \in \X \times \Y$. Therefore the disintegration of $\pi$ against its $\mu$-marginal at $x$ is given by $p(x,\cdot) \cdot \nu$ and $p(x,\cdot)$ is the probability density of $Y$ with respect to $\nu$, when conditioned on $X=x$.

Now, let $(X_j,Y_j)$, $j \in [M]$, be i.i.d.~random variables with law $\pi$.
We set 
$(\zb X, \zb Y) \coloneqq \left((X_j,Y_j)\right)_{j=1}^M$.  
Similarly, 
for $(x_j,y_j) \in \X \times \Y$, $j \in [M]$, 
we set
$(\zb x,\zb y) \coloneqq \left((x_j,y_j)\right)_{j=1}^M$. 
Then the joint law of $(\zb X, \zb Y)$ and its probability density with respect to $(\mu \otimes \nu)^{\otimes M}$ are given by
\[
P_{(\zb X,\zb Y)} = \pi^{\otimes M} = (p \cdot (\mu \otimes \nu))^{\otimes M} = p^{\otimes M} \cdot (\mu \otimes \nu)^{\otimes M}.
\]                                                 
Unfortunately, according to Assumption \ref{asp:main} we cannot directly sample from $(\zb X, \zb Y)$, but have to take permutations 
in the second component into account.
To model this, let $\mathcal G_M$ denote the permutation group on $[M] \coloneqq \{1,\ldots,M\}$.
Recall that $|\mathcal G_M| = M!$.
Let $\Theta: (\Omega,\Sigma,\mathbb P) \to \mathcal G_M$ be a discrete random variable which is uniformly distributed on
$\mathcal G_M$, i.e. $P(\Theta = \sigma) = \frac{1}{M!}$ for every $\sigma \in \mathcal G_M$, and independent of $(\zb X, \zb Y)$.
For $\sigma \in \mathcal G_M$ 
we denote the associated permutation matrix by $M_\sigma$.
For $\Theta(\omega) = \sigma$, we set $\Theta_j(\omega) \coloneqq \sigma(j)$.
Now our samples arise from concatenating the random variables $\bm{X},\bm{Y}$ and $\bm{\Theta}$.
To this end, let
\[
S: (\mathbb X \times \mathbb Y)^M \times \mathcal G_M \to (\mathbb X \times \mathbb Y)^M , \quad
((\zb x, \zb y), \sigma) \mapsto ((x_j,y_{\sigma(j)}))_{j=1}^M
\]
and $S_\sigma \coloneqq S(\cdot,\cdot,\sigma): (\mathbb X \times \mathbb Y)^M \to (\mathbb X \times \mathbb Y)^M$
for any $\sigma \in \mathcal G_M$.
Then our sampling procedure can be considered as a realization 
of the random variable
\begin{equation}\label{Z}
\zb Z = S \circ \left((\zb X,\zb Y),\Theta \right): (\Omega,\Sigma,\mathbb P) \to   (\X \times \Y)^M, \quad
\omega \mapsto \left( \left(X_j(\omega),Y_{\Theta(\omega)(j)} (\omega) \right)\right)_{j=1}^M.
\end{equation}
As we will see below, 
the distribution of $\zb{Z}$ can be formulated 
via the \emph{symmetrization operator on functions}
\begin{equation}\label{sym_1}
s_M: \mathcal C(\X\times \Y) \to\ \mathcal C( (\X \times \Y)^M ), \quad s_M p\assign \frac{1}{M!} \sum_{\sigma \in \mathcal G_M} p^{\otimes M} \circ S_\sigma,
\end{equation}
i.e.\
\begin{equation}
s_M p (\zb x,\zb y) \assign 
\frac{1}{M!} \sum_{\sigma \in \mathcal G_M} 
p_{M_\sigma(\zb Y)|\zb X = \zb x} (\zb y)
= \frac{1}{M!} \sum_{\sigma \in \mathcal G_M} \prod_{j=1}^M p(x_j, y_{\sigma(j)} ).
\end{equation}

Clearly, we have for any $\sigma \in \mathcal G_M$ the invariance
\begin{equation} \label{sym}
(s_M p) \circ S_\sigma= s_M p.
\end{equation}

\begin{proposition} \label{prop:basic}
The law of the random variable $\zb Z \in (\mathbb X \times \mathbb Y)^M$ in \eqref{Z} is given by
\begin{equation} \label{law_Z}
P_{\zb Z} = p_{\zb Z} \cdot (\mu \otimes \nu)^{\otimes M} ,
\quad \text{with} \quad 
p_{\zb Z}  \coloneqq s_M p .
\end{equation}
\end{proposition}

\begin{proof}
By the law of total probability, we obtain
\begin{align} \label{start}
P_{\zb Z}
= \sum_{\sigma \in \mathcal G_M} P_{\zb Z|\Theta = \sigma} \, P(\Theta = \sigma)
= 
\frac{1}{M!}\sum_{\sigma \in \mathcal G_M} P_{\zb Z|\Theta=\sigma}.
\end{align}
For the law of $\zb Z$ when conditioned on $\Theta=\sigma$, we have 
\[
P_{\zb Z|\Theta=\sigma} = P_{S \circ ((\zb X,\zb Y),\Theta)|\Theta=\sigma} = P_{S_\sigma(\zb X,\zb Y)} 
=
(S_\sigma)_\# (\pi^{\otimes M})
=
(S_\sigma)_\# \left(p^{\otimes M} \cdot (\mu \otimes \nu)^{\otimes M} \right).
\]
Obviously, $S_\sigma: (\mathbb X \times \mathbb Y)^M \to (\mathbb X \times \mathbb Y)^M$ is a diffeomorphism 
which leaves $(\mu \otimes \nu)^{\otimes M}$ unchanged, i.e.,
\begin{equation} \label{fix}
(S_\sigma)_\#  (\mu \otimes \nu)^{\otimes M} = (\mu \otimes \nu)^{\otimes M}.
\end{equation}
Moreover, it holds $\text{det} (\nabla S_\sigma) = \text{det} (M_\sigma) \in \{\pm 1\}$.
Thus, together with the change of variables formula, we get
\begin{equation}\label{eq:LawZConditioned}
P_{\zb Z|\Theta=\sigma}= p^{\otimes M} \circ S_\sigma^{-1} \cdot (\mu \otimes \nu)^{\otimes M} 
\end{equation}
Inserting this into \eqref{start} and recalling the definition \eqref{sym_1} (where we use that $S_\sigma^{-1}=S_{\sigma^{-1}}$ and relabel the sum), we obtain the assertion.
\end{proof}

In general, we cannot find any arbitrary $p$ from the negative log-likelihood $-\log(s_M p)$.
Instead, we will only search for a hypothesis density $q \in Q \subset \mathcal C_+(\X \times \Y)$
from a parametric or nonparametric space $Q$ 
with corresponding measure $q \cdot (\mu \otimes \nu)$.
Later, we will model $Q$ as a nonparametric space using kernels from entropic optimal transport.

\begin{definition}[Permutation inference functional] \label{def1}
We define the {\rm population permutation functional} as
the expectation value of $-\log (s_M q)$ with respect to $P_{\bm{Z}}$, that is
\begin{align}\label{jm}
\mathcal J_{M}(q) & \assign - \int_{(\X \times \Y)^M} \log(s_M q) \, \diff  P_{\bm{Z}}.
\end{align}
For $N$ independently drawn samples $
(\zb x^i,\zb y^i)
= \left((x_j^i,y_j^i) \right)_{j=1}^M, i \in [N]
$
of $\zb Z$, 
the {\rm empirical permutation functional} is given by 
\begin{align}  \label{true_log}
\mathcal J_{M}^N (q)
&= - \int_{(\X \times \Y)^M} \log(s_M q) \, \diff \pi_M^N
= - \frac1N \sum_{i=1}^N \log 
\Big( (s_M q) \big((\zb x^i, \zb y^i) \big) \Big)
\end{align}
with the empirical measure  
\begin{equation}  \label{empirical}
\pi_M^N \coloneqq \frac1N \sum_{i=1}^N \delta_{(\zb x^i, \zb y^i)}.
\end{equation}
\end{definition}

Interestingly, we can rewrite the population permutation functional in another way.

\begin{lemma}\label{equal1}
    The functional $\mathcal J_M$ in \eqref{jm} admits the equivalent form
    \begin{equation*}
        \mathcal J_{M}(q) 
        = 
        - \int_{(\X \times \Y)^M} \log(s_M q) \, \diff \pi^{\otimes M}.
    \end{equation*}
\end{lemma}

\begin{proof}
    By \eqref{fix}, we have
    \begin{align*}
        &-\mathcal J_{M}(q)
                =
        \int_{(\X \times \Y)^M} 
        \log(s_M q)
        \,\diff P_{\bm{Z}}
        \\
        =
        &\int_{(\X \times \Y)^M} 
        \log(s_M q)\cdot
        (s_M p)
        \,\diff (\mu \otimes \nu)^{\otimes M}
        \\
        =
        &
        \frac{1}{M!} \sum_{\sigma \in \mathcal{G}_M}
        \int_{(\X \times \Y)^M} 
        \log\bigl(
        s_M q
        \bigr)\cdot
        \big(p^{\otimes M} \circ S_\sigma\big)
        \,\diff (\mu \otimes \nu)^{\otimes M}
        \\
        =
        &
        \frac{1}{M!} \sum_{\sigma \in \mathcal{G}_M}
        \int_{(\X \times \Y)^M} 
        \log
        \bigl(
        \underbrace{
        (s_M q) \circ S_\sigma^{-1}
        }_{=s_M q \textnormal{ by \eqref{sym}}} 
        \bigr)\cdot
        \big(p^{\otimes  M}\big)
        \,\diff (\mu \otimes \nu)^{\otimes M}
        \\
         =
        &
        \int_{(\X \times \Y)^M} \log(s_M q) \, \diff \pi^{\otimes M}.
    \end{align*}   
    This yields the assertion.
\end{proof}

Next, let us motivate these functionals from the point of view of the KL divergence.
We introduce the \emph{symmetrization operator on measures} 
\[
\mathcal S_M : \mathcal P((\X \times \Y)^M) \to \mathcal P((\X \times \Y)^M), \quad
\mathcal S_M{\boldsymbol \xi} \coloneqq \frac{1}{M!} \sum_{\sigma \in \mathcal G_M} S_{\sigma\#} \boldsymbol{\xi}, \quad 
\boldsymbol{\xi} \in \mathcal P((\X \times \Y)^M).
\]
The relation between $S_M$ and $s_M$ becomes clear in the following proposition and leads to the desired KL
characterization.

\begin{proposition} \label{prop:1}
For $p,q \in \mathcal C(\X \times \Y)$
we consider the measure $\pi\coloneqq p \cdot (\mu \otimes \nu)$ and $\gamma \coloneqq q \cdot (\mu \otimes \nu)$.
Then the following holds true:
\begin{itemize}
\item[i)] The measure $\mathcal S_M \pi ^{\otimes M}$ is absolutely continuous 
with respect to $(\mu \otimes \nu)^{\otimes M}$ and
\begin{equation}
\mathcal S_M \pi ^{\otimes M} = (s_M p) \cdot (\mu \otimes \nu)^{\otimes M}.
\end{equation}
The same holds true for $\pi ^{\otimes M}$ and in particular, 
we have $P_{\zb Z} = \mathcal S_M \pi^{\otimes M}$.
\item[ii)]
The functional $\mathcal J_{M}$ in \eqref{jm} can be written as
\begin{equation}
\mathcal J_{M}(q)  \coloneqq \KL(\mathcal S_M \pi^{\otimes M}| \mathcal S_M \gamma^{\otimes M}) + \text{const}
\end{equation}
with a constant not depending on $q$.
\end{itemize}
\end{proposition}

\begin{proof}
\begin{itemize}
    \item[i)] Observing that $(\mu \otimes \nu)^{\otimes M}$ is invariant under $\mathcal{S}_M$,
    we see as in the proof of Proposition \ref{prop:basic} that
    \[
    \mathcal S_M \gamma^{\otimes M} 
    = \mathcal S_M \big(q^{\otimes M} \cdot (\mu \otimes \nu)^{\otimes M} \big) 
    = (s_M q) \cdot (\mu \otimes \nu)^{\otimes M}.
    \]
    For $\pi$ this relation implies (see again Proposition \ref{prop:basic}) that $P_{\zb Z} = \mathcal S_M \pi^{\otimes M}$.
    \item[ii)] Due to i) it holds
    $\mathcal S_M \pi^{\otimes M} = (s_M p) \cdot (\mu \otimes \nu)^{\otimes M}$ 
    and 
    $\mathcal S_M \gamma^{\otimes M} = (s_M q) \cdot (\mu \otimes \nu)^{\otimes M}$.
    Thus we obtain by \eqref{kl}
    \[
    \KL(\mathcal S_M \pi^{\otimes M}| \mathcal S_M \gamma^{\otimes M}) 
    = \int_{(\X \times \Y)^M} \log (s_M p) \, \text{d} \pi^{\otimes M} 
    - \int_{(\X \times \Y)^M} \log (s_M q) \, \text{d} \pi^{\otimes M}.
    \]
    The first summand is constant with respect to $q$ and the second is $\mathcal J_M(q)$.
\end{itemize}
\end{proof}

%---------------------------------------------
\subsection{Approximate inference functional} \label{subsec:32}
%---------------------------------------------

Unfortunately, as $M$ increases, inferring the density $q$ via $\mathcal J_{M}^N$ very quickly 
becomes computationally intractable due to the large number of possible permutations. 
Therefore, we need to resort to an approximate functional, which we motivate in the following.

We have that
\begin{equation}
s_M p (\zb x, \zb y) = \frac{1}{M!} \sum_{\sigma \in \mathcal G_M} 
p_{M_\sigma(Y)|\zb X = \zb x} (\zb y)
=
p_{Y_{\Theta_1},...,Y_{\Theta_M}|\zb X=\zb x}(\zb y).
\end{equation}
Of course, $(Y_{\Theta_j} |\zb X = \zb x)_{j=1}^M$
are not independent. 
But if we allow for this approximation,
we get
\begin{equation}\label{prob_approx}
s_M p (\zb x, \zb y) \approx \prod_{j=1}^M p_{Y_{\Theta_j}|\zb X=\zb x}(y_j).
\end{equation}
The  factors in \eqref{prob_approx} are given by the following lemma.

\begin{lemma} \label{lem:factors}
It holds 
$$
p_{Y_{\Theta_j} \vert \bm{X} = \bm{x}}(y) 
= (t_Mp)(\zb x,y),
$$
where the operator $t_M: \mathcal C(\X\times \Y) \to\ \mathcal C(\X^M \times \Y)$ 
is given by
$$
(t_Mq)(\zb x,y) := \frac{1}{M}\sum_{k=1}^M q(x_k,y), 
$$
\end{lemma}

\begin{proof}
For identification purposes, let $\X_i,\Y_i$, $i \in [M]$, 
be copies of $\X,\Y$, respectively.
The probability for $(\bm{X},Y_{\Theta_j})$ given $\Theta = \sigma$ to lie within some measurable set $A \subseteq \X^N \times \Y_{\sigma(j)}$ is, using \eqref{eq:LawZConditioned},
\begin{align*}
&P_{\bm{X},Y_{\Theta_j} \vert \Theta = \sigma}(A)
=
P_{\bm{Z} \vert \Theta = \sigma} (A \times \bigtimes_{l \neq \sigma(j)} \Y_l)
\\
=
&\int_{A \times \bigtimes_{l \neq \sigma(j)} \Y_l} 
\prod_{l} p(x_l,y_{\sigma^{-1}(l)}) \,
\diff (\mu \otimes \nu)^{\otimes M} (\bm{x},\bm{y})
\\
=
&\int_{A} 
p(x_{\sigma(j)},y_j) 
\underbrace{\biggl(
\prod_{l \neq \sigma(j)} 
\int_{\Y} p(x_l,y_{\sigma^{-1}(l)}) \diff \nu(y_{\sigma^{-1}(l)})
\biggr)}_{=1 \textnormal{ by \eqref{eq:pMarginals}}}
\diff (\mu^{\otimes M} \otimes \nu) (x_1,\dotsc,x_M,y_{j}).
\end{align*}
Hence we conclude
\[
P_{\bm{X},Y_{\Theta_j} \vert \Theta = \sigma} 
= 
p_{\bm{X},Y_{\Theta_j} \vert \Theta = \sigma} \cdot \mu^{\otimes M} \otimes \nu
\quad \text{with} \quad
p_{\bm{X},Y_{\Theta_j} \vert \Theta = \sigma} (\zb x, y)
\coloneqq
p(x_{\sigma(j)},y).
\]
By \eqref{eq:pMarginals}, the conditional density of $Y_{\Theta_j}$ 
with respect to $\nu$, 
conditioned on $\bm{X} = \bm{x}$ and $\Theta = \sigma$,
is given by 
$p_{Y_{\Theta_j} \vert \bm{X} = \bm{x}, \Theta = \sigma}(y) = p(x_{\sigma(j)},y)$.
Finally, the conditional density of $Y_{\Theta_j}$ 
with respect to $\nu$, 
conditioned on $\bm{X} = \bm{x}$ reads as
\begin{align*}
p_{Y_{\Theta_j} \vert \bm{X} = \bm{x}}(y) 
&= \sum_{\sigma \in \mathcal{G}_M} p_{Y_{\Theta_j} \vert \bm{X} = \bm{x}, \Theta = \sigma}(y) \cdot P(\Theta = \sigma)
= \frac{1}{M!} \sum_{\sigma \in \mathcal{G}_M} p(x_{\sigma(j)},y)
\\
&= \frac{1}{M!} \sum_{k=1}^M \sum_{\sigma \in \mathcal{G}_M : \sigma(j) = k} p(x_k,y)
= \frac{1}{M} \sum_{k=1}^M p(x_k,y). \qedhere
\end{align*}
\end{proof}

Using the lemma, 
in \eqref{prob_approx} we get
\[
s_M p (\zb x, \zb y) 
\approx \prod_{j=1}^M \frac1M \sum_{k=1}^M p(x_k,y_j).
\]
Substituting this into the permutation functional $\mathcal J_M$,
we obtain, up to a factor $\frac{1}{M}$, 
the approximation functional below. 
We include this factor, 
as it will lead to more convenient expressions in our further analysis.

\begin{definition}[Approximate inference functional] \label{def2}
We define the {\rm inference functional}  by
\begin{align} 	
	J_M(q) 
 & \assign - \frac{1}{M} \int_{(\X \times \Y)^M}  \sum_{j=1}^M 
	\log \Big( \frac1M
	\sum_{k=1}^M q(x_{k},y_{j}) 
	\Big)
	\diff P_{\zb Z}(\zb x, \zb y) \\
 &= - \frac{1}{M} \int_{(\X \times \Y)^M}  \sum_{j=1}^M 
	\log \Big( \frac1M
	\sum_{k=1}^M q(x_{k},y_{j}) 
	\Big)
	\diff \pi^{\otimes M} (\zb x, \zb y) . 
 \label{eq:LossCloudsPopulation}
	\end{align}
For $N$ independently drawn samples \smash{$
(\zb x^i,\zb y^i)
= \left((x_j^i,y_j^i) \right)_{j=1}^M, i \in [N]
$}
of $\zb Z$,  the {\rm empirical inference functional} is given by
\begin{equation} 	\label{eq:LossCloudsEmpirical}
	J^N_M(q) \assign - \frac{1}{M N} \sum_{i=1}^N \sum_{j=1}^M \log \Big(
	\frac1M \sum_{k=1}^M q(x_{k}^i,y_{j}^i) \Big),
\end{equation}
\end{definition}

We remark that, the equality in \eqref{eq:LossCloudsPopulation} 
can be seen in the same way as Lemma \ref{equal1}.
The approximate inference functionals can be written in a simpler form. 

\begin{lemma}\label{jm_short}
The functional $J_M$ in \eqref{eq:LossCloudsPopulation} can be written as
\begin{align} \label{jm_other}
J_M(q)  &= - \int_{\X^M \times \Y} \log( t_M q) \, \diff \pi \diff \mu^{\otimes (M-1)}.
\end{align}
\end{lemma}

\begin{proof}
We can rewrite $J_M$ as
\begin{align}
&J_M(q) 
=
- \frac{1}{M} \sum_{j=1}^M 
\int_{(  \X \times \Y)^M} 
\log \big(
	(t_M q)(\zb x,y_j)
	\big)
	\diff \pi(x_1,y_1)\ldots \diff \pi(x_M,y_M) \\
& =
-\frac1M	\sum_{j=1}^M \int_{  \X^{M} \times \Y}
\log \big(
	(t_M q)(\zb x,y_j)
	\big)
	\diff \mu(x_1)  \ldots \diff \mu(x_{j-1}) \diff \pi(x_j,y_j) \diff \mu(x_{j+1}) \ldots \diff \mu(x_M)\\
&=	- \int_{\X^M \times \Y} \log \big(
	(t_M q)(\zb x,y_1)
	\big) \diff \pi(x_1,y_1) \diff \mu(x_2)  \ldots \diff \mu(x_M)
\end{align}
which yields the assertion.
\end{proof}

Next, we introduce  the projection operator 
$P : (\X \times \Y)^M \to (\X^M \times \Y)$ by
\begin{equation}
P\left((\zb x,\zb y)\right) \mapsto (\zb x,y_1).
\end{equation}
By the following proposition,
$t_M p$ plays the same role for $P_\#S_M \pi^{\otimes M}$ as
$s_M p$ did for $S_M \pi^{\otimes M}$.

\begin{proposition}\label{prop:KL_form}
For $p,q \in \mathcal C(\X \times \Y)$
we consider the measure $\pi\coloneqq p \cdot (\mu \otimes \nu)$ and $\gamma \coloneqq q \cdot (\mu \otimes \nu)$.
Assume that $\int_\Y q(x,y) \, \diff\nu(y)=1$ for $\mu$-almost all $x \in \X$.
Then the following holds true:
\begin{itemize}
\item[i)] The measure $P_\# \mathcal S_M \pi ^{\otimes M}$ is absolutely continuous 
with respect to $(\mu^{\otimes M} \otimes \nu)$ and
\begin{equation}
P_\# \mathcal S_M \gamma^{\otimes M} = (t_M q) \cdot (\mu^{\otimes M} \otimes \nu).
\end{equation}
\item[ii)]
The functional $J_{M}$ in \eqref{eq:LossCloudsPopulation} can be written as
\begin{equation}
J_{M}(q)  \coloneqq \KL( P_\# \mathcal S_M \pi^{\otimes M}|  P_\# \mathcal S_M \gamma^{\otimes M}) + \text{const}
\end{equation}
with a constant not depending on $q$.
\end{itemize}
\end{proposition}

\begin{proof}
i) For any $\phi \in \mathcal C(\X^M\times \Y)$, we obtain
\begin{align*}
  & \int_{\X^M\times \Y} \phi \, \diff P_{\#} S_M \gamma^{\otimes M}
  = \frac{1}{M!}\sum_{\sigma \in \mathcal G_M} \int_{(\X \times \Y)^M} \phi\circ P\circ S_\sigma \ \diff \left(q\cdot \mu\otimes \nu\right)^{\otimes M}
	\\
  = & \frac{1}{M!}\sum_{\sigma \in \mathcal G_M}
	\int_{(\X \times \Y)^M} \phi(\zb x,y_{\sigma(1)}) \cdot \prod_k q(x_k,y_k)\,
    \diff \mu(x_1)\ldots\diff \mu(x_M) \diff \nu(y_1)\ldots\diff \nu(y_M)
\intertext{and further by assumption on $q$ that}
= &
  \frac{1}{M!}\sum_{i=1}^M \sum_{\sigma:\sigma(1)=i} \int_{\X^M\times \Y} \phi(\zb x,y_{i}) \cdot q(x_i,y_i)\,
    \diff \mu^{\otimes M}(\zb x)  \diff \nu(y_i) 
		\\
  = & \frac{1}{M}\sum_{i=1}^M \int_{\X^M\times \Y} \phi(\zb x,y_{i}) \cdot q(x_i,y_i)\,
    \diff \mu^{\otimes M}(\zb x) \diff \nu(y_i) \\
  = & \int_{\X^M\times \Y} \phi(\zb x,y_{1}) \cdot \Big( \frac{1}{M}\sum_{i=1}^M q(x_i,y_1) \Big)\,
    \diff \mu^{\otimes M}(\zb x) \diff \nu(y_1) \\
  = & \int_{\X^M\times \Y} \phi \cdot (t_M q)\,\diff (\mu^{\otimes M}\otimes \nu).
\end{align*}
ii) By \eqref{kl} and part i) we conclude 
$$
\KL(P_{\#}S_M \pi^{\otimes M}| P_{\#}\mathcal S_M \gamma^{\otimes M}) 
= \int_{\X^M \times \Y} \!\!\!\!\!\log (t_M p) \, \diff P_{\#}S_M \pi^{\otimes M} 
- \int_{\X^M \times \Y} \!\!\!\!\!\log (t_M q) \, \diff P_{\#}S_M \pi^{\otimes M}.
$$
The first summand is constant with respect to $q$.
For the second summand, we get 
\begin{align*}
&	-\int_{\X^M \times \Y} \log( t_M q)\,\diff P_\# \mathcal S_M \pi^{\otimes M}
= 	- \frac{1}{M!} \sum_{\sigma \in \mathcal G_M} \int_{(\X \times \Y)^M} 
\log( (t_M q) \circ P \circ S_\sigma) \, 
\diff \pi^{\otimes M}
\\
&= - \frac{1}{M!} \sum_{\sigma \in \mathcal G_M} \int_{(\X \times \Y)^M} 
\log((t_M q)(\zb x,y_{\sigma(1)})) \, \diff \pi(x_1,y_1)\ldots \diff \pi(x_M,y_M)   \\
&= - \frac{1}{M!} \sum_{j=1}^M \sum_{\sigma:\sigma(1)=j} \int_{\X^M \times \Y} \log((t_M q)(\zb x,y_{j}))\,\\
& \qquad 
	\diff \mu(x_1)\ldots\diff \mu(x_{j-1}) 
	\diff \pi(x_j,y_j)
	\diff \mu(x_{j+1})\ldots \diff \mu(x_M) \\
&=  - \int_{\X^M \times \Y} \log(t_M q)(\zb x,y_1) \, \diff \pi(x_1,y_1)\diff \mu(x_2)\ldots\diff \mu(x_M). 
\end{align*}
which is by Lemma \ref{jm_short} the same as $J_{M}(q)$. 
\end{proof}

Finally, let $\bm{Z}_i$, $i \in [N]$, 
be independently distributed random variables as $\bm{Z}$ in \eqref{Z}.
We consider
\[
\pi_M^N:\Omega \to \mathcal{P}((\X \times \Y)^M), \quad
\pi_M^N(\omega) 
\coloneqq 
\frac{1}{N} \sum_{i=1}^N \delta_{\bm{Z}_i(\omega)}.
\]
Note that $\pi_M^N$ is a random measure,
that is the map
\[
\Omega \to [0,\infty), \,\, \omega \mapsto \big(\pi_M^N(\omega)\big)(A)
\]
is measurable for all Borel sets $A \in (\X \times \Y)^M$.
For simplicity, and as is common, 
we use the same notation for the random measure $\pi_M^N$
and for the associated empirical measure in \eqref{empirical}.
The meaning is always clear from the context.
For the next result, we consider $\mathcal J_M^N(q)$ and $J_M^N(q)$
as random variables obtained when replacing the empirical $\pi_M^N$
with the random version in their respective definition.

\begin{corollary}\label{lem:SymOp}
Let $q \in \mathcal C(\X \times \Y)$.
\begin{itemize}
\item[i)] The expected value $\mathbb E(\pi_M^N) \in \mathcal P((\X \times \Y)^M)$ of $\pi_M^N$
    is given by
    \[
    \int_{\Omega} \pi^N_M(\omega) \, \diff \PP(\omega)
    = 
    \mathcal{S}_M \pi^{\otimes M}, 
        \]
    where the integral on the left-hand side is a Bochner (Pettis) integral.
    It holds that
    $$
    \pi_M^N \rightweaks \mathcal{S}_M \pi^{\otimes M} \text{ as }  N \to \infty \; \text{a.s.}.
    $$
\item[ii)]  The expected value of $\mathcal J_M^N(q)$ is $\mathcal J_M(q)$ and 
    $\mathcal J_M^N(q) \to \mathcal J_M(q)$ a.s. as $N \to \infty$.
\item[iii)] 
    The expected value of $J_M^N(q)$ is $J_M(q)$ and 
    $J_M^N(q) \to J_M(q)$ a.s. as $N \to \infty$.
    \end{itemize}
\end{corollary}

\begin{proof}
\begin{itemize}
    \item[i)] 
    First, since $(\bm{Z}_i)_{i=1}^N$ are independent
    identically distributed,
    the same holds true for the random variables $(\delta_{\bm{Z}_i}(A))_{i=1}^N$
    for any fixed measurable $A \subseteq (\X \times \Y)^M$.
    Together with the definition of the Bochner integral,
    this gives
    \[
    \biggl(
    \int_{\Omega} \pi^N_M(\omega) \, \diff \PP(\omega)
    \biggr)(A)
    =
    \int_{\Omega} \big(\pi^N_M(\omega) \big)(A) \, \diff \PP(\omega)
    =
    \int_{\Omega} \delta_{\bm{Z}(\omega)}(A) \, \diff \PP(\omega).
    \]
    Proceeding from there, we obtain with Proposition \ref{prop:1}i) that
    \begin{align*}
        \int_{\Omega} \delta_{\bm{Z}(\omega)}(A) \, \diff \PP(\omega)
        &= 
        \int_{(\X \times \Y)^M} 
        \delta_{(\zb x, \zb y)}(A) 
        \, \diff P_{\bm{Z}}(\zb x, \zb y)
        =
        \int_{A} 
        \, \diff \mathcal{S}_M \pi^{\otimes M}(\zb x, \zb y).
    \end{align*}
    Furthermore, 
    for any fixed measurable $A \subseteq (\X \times \Y)^M$,
    $(\delta_{\bm{Z}_i}(A))_{i=1}^N$ 
    is a family of independent, identically distributed
    real-valued random variables, so that the
    strong law of large numbers yields
    \[
    \pi_M^N(A) 
    \to \int_{\Omega} \delta_{\bm{Z}(\omega)}(A) \, \diff \PP(\omega) 
    = \mathcal{S}_M \pi^{\otimes M}(A), \quad \text{a.s.}, 
    \]
    as $N \to \infty$.
    The Portmanteau theorem implies
    $\pi_M^N \rightweaks \pi_M^N$ almost surely.

    \item[ii)]
    By part i), we  obtain
    \begin{align*}
        \mathbb{E}(\mathcal{J}_M^N(q)) 
        &= \int_{\Omega} - \int_{(\X \times \Y)^M} \log(s_M q) \, \diff \pi_M^N \, \diff \PP
        \\
        &= - \int_{(\X \times \Y)^M} \log(s_M q) \, \diff \mathcal{S}_M \pi^{\otimes M}
        = \mathcal{J}_M(q).
    \end{align*}
    Using i) again together with $\log(s_M q) \in C((\X \times \Y)^M)$ yields
    $\mathcal J_M^N(q) \to \mathcal J_M(q)$ almost surely as $N \to \infty$,
    as desired.

    \item[iii)] This follows by the same arguments as ii). \qedhere
\end{itemize}
\end{proof}

%----------------------------------------------------------
\subsection{Basic properties of inference functionals}\label{sec:FunctionalsBasic}
%-----------------------------------------------------------
In the following, we prove various basic properties of our inference functionals.

\begin{proposition}[Lower semi-continuity]\label{prop:lsc}
The functionals 
$J_M, J_M^N: \mathcal C_+(\X \times \Y) \to \mathbb R$ defined by \eqref{eq:LossCloudsPopulation} and \eqref{eq:LossCloudsEmpirical}
are lower semi-continuous.
If $q>0$, then $J_M$ and $J_M^N$ are continuous in $q$.
\end{proposition}

\begin{proof}
We give the proof for $J_M$. For $J_M^N$ the assertion follows by the same arguments.

Let $q \in \mathcal C_+(\X \times \Y)$ and $(q_k)_k$ be a sequence of non-negative continuous functions with $\norm{q_k-q}_\infty\to 0$ as $k \to \infty$.
\\
1. First, assume that $q >0$. 
By compactness of $\X \times \Y$ there exist 
$\varepsilon >0$ and  $K \in \mathbb N$ such that 
$q, q_k \geq \varepsilon$ for all $k \geq K$. 
Then we also have $t_M q_k, t_M q \geq \varepsilon > 0$.
Since $t_M$ is continuous and the logarithm is continuous on $(0,\infty)$, we obtain $-\log (t_M q_k) \to -\log (t_M q)$ uniformly 
and then $J_M(q_k) \to J_M(q)$  as $k \to \infty$.
\\
2. Now we consider a function $q \ge 0$.
For $l \in \mathbb N$, set 
$$q_k^l: (x,y) \mapsto \max\{q_k(x,y),1/l\} 
\quad \text{and} \quad  
q^l: (x,y) \mapsto \max\{q(x,y),1/l\}.
$$
Then we have $q^l \geq q$, $q^l_{k} \geq q_{k}$, and further
$t_M q^l \geq t_M q$, $t_M q^l_k \geq t_M q_k$, so that
$-\log(t_M q^l) \nearrow -\log(t_M q)$ pointwise. 
Since $q\in\mathcal C(\X \times \Y)$, 
the function $q$ is bounded from above. Then also $-\log(t_m q)$ is bounded from below and all $-\log(t_M q^l)$ are equi-bounded from below. Now the monotone convergence theorem implies 
\begin{equation} \label{h1}
\lim_{l \to \infty} J_M(q^l) = J_M(q).
\end{equation}
Next we conclude that
$$
J_M(q_k) \geq J_M(q_k^l) = J_M(q)
+ \left(J_M(q^l)-J_M(q)\right)
+ [\left( J_M(q_k^l)-J_M(q^l) \right).
$$
For fixed $l \in \mathbb N$ and $k \to \infty$ the last summand goes to zero by Part 1 of the proof.
Hence we get
$$\liminf_{k \to \infty} J_M(q_k) \geq J_M(q)
+[J_M(q^l)-J_M(q)].
$$
Finally, sending $l \to \infty$ and using \eqref{h1} we obtain the lower semi-continuity of $J_M$.
\end{proof}

From Proposition~\ref{prop:lsc} and Corollary~\ref{lem:SymOp} we deduce the following result.

\begin{corollary}[Joint convergence]\label{prop:JointLimit}
Let $q \in \mathcal C_+(\X \times \Y)$ with $q > 0$ 
and  let $(q^N)_N$ be a sequence in $\mathcal C_+(\X \times \Y)$ which converges uniformly to $q$. Then it holds almost surely that $J^N_M(q^N) \to J_M(q)$ as $N \to \infty$.
\end{corollary}

\begin{proof}
It holds
\begin{align*}
\lvert J_M^N(q^N) - J_M(q) \rvert 
&\leq \lvert J_M^N(q^N) - J_M^N(q) \rvert + \lvert J_M^N(q) - J_M(q) \rvert.
\end{align*}
For $N \to \infty$, both terms converge to $0$ almost surely. 
Indeed, on one hand it holds
\begin{align*}
\lvert J_M^N(q^N) - J_M^N(q) \rvert 
\leq &\int_{(\X \times \Y)^M}\lvert\log (t_M q_N) - \log (t_M q) \rvert \, \diff P_{\bm{Z}}
\\
\leq &\|\log (t_M q_N) - \log (t_M q)\|_\infty \to 0, \quad N \to \infty,
\end{align*}
as in the proof of Proposition~\ref{prop:lsc}.
On the other hand, we have
$\lvert J_M^N(q) - J_M(q) \rvert \to 0$, $N \to \infty$, almost surely,
due to Corollary~\ref{lem:SymOp}.
\end{proof}

Next, we give a  basic $\Gamma$-convergence-type result that establishes that by minimizing the empirical functionals $J_M^N$, we can recover the population minimizer in the limit of $N \to \infty$. For simplicity, we assume that the candidate densities are bounded away from zero, which will be sufficient for the class of hypothesis densities that we introduce in Section \ref{sec:Kernels}.

Recall that  a sequence $(F_N)_{N\in\N}$ 
of functionals 
$F_N \colon \mathcal{C}(\X \times \Y) \rightarrow (-\infty,+\infty]$ 
is said to $\Gamma$-converge to 
$F \colon  \mathcal{C}(\X \times \Y)  \rightarrow (-\infty,+\infty]$,
if the following two conditions are fulfilled for every $q \in \mathcal{C}(\X \times \Y) $, see~\cite{Braides02}:
\begin{enumerate}
	\item[i)] It holds 
    $F(q) \leq \liminf_{N \rightarrow \infty} F_N(q_N)$ 
    whenever $q_N \to q$ as $N \to \infty$.
	\item[ii)] There exists a sequence $(q_N)_{N\in\N}$ with $q_N \to q$ 
    and $\limsup_{N \to \infty} F_N(q_N) \le F(q)$.
\end{enumerate}
The importance of $\Gamma$-convergence lies in the fact that 
every cluster point of minimizers of $\{F_N\}_{N\in\N}$ is a minimizer of $F$, and any minimizer of $F$ can be approximated by a sequence of almost-minimizers of the $\{F_N\}_{N\in\N}$.

\begin{corollary}[$\Gamma$-convergence]\label{cor:Gamma}
Let $Q$ be a compact subset of $\mathcal C (\X \times \Y)$ 
which is uniformly bounded away from zero, 
i.e.~there exists some $\delta>0$ 
such that $q \geq \delta$ for all $q \in Q$.
Furthermore, let $(Q^N)_N$ be a sequence of subsets of $\mathcal C(X \times Y)$ 
such that the $\Gamma$-limit of $(\iota_{Q^N})_N$ 
is $\iota_{Q}$.
Then the $\Gamma$-limit of $(J_M^N+ \iota_{Q^N})_N$ is $J_M + \iota_Q$ almost surely.
If the sets $(Q^N)_N$ are compact, 
then the problems $(J_M^N+ \iota_{Q^N})_N$ have minimizers $(q_N)_N$. 
Almost surely, any cluster point of $(q_N)_N$ is a minimizer of $J_M + \iota_Q$.
\end{corollary}

Before proceeding with the proof,
we discuss the conditions of $\Gamma$-convergence
for the $\iota$-functions.
\begin{remark}\label{rem:Gamma_convergence_for_iota}
Let $F_N = \iota_{Q^N}$, $F = \iota_{Q}$.
If $q \in Q$, 
then the first condition of $\Gamma$-convergence 
is always fulfilled.
In the same manner, if $q \not\in Q$,
then the second condition is always fulfilled.
Hence, $(\iota_{Q^N})_N$ $\Gamma$-converges to $\iota_{Q}$,
if and only if
\begin{enumerate}
    \item[i)] 
    For all $q \not\in Q$ and any $q_N \to q$, $N \to \infty$,
    there exists $\bar{N} \in \N$ so that $q_N \not \in Q_N$,
    for all $N \geq \bar{N}$.
    \item[ii)]
    For all $q \in Q$ there exists $q_N \xrightarrow[]{N \to \infty} q$, $N \to \infty$,
    and $\bar{N} \in \N$ so that $q_N \in Q_N$,
    for all $N \geq \bar{N}$.
\end{enumerate}
\end{remark}

\begin{proof}
Define $F \coloneqq J_M + \iota_Q$ and $F_N \coloneqq J_M^N + \iota_{Q^N}$, $N \in \N$.
To check the conditions for $\Gamma$-convergence, 
let $q \in Q$.
First, consider an arbitrary sequence $(q^N)_N \subset C(\X \times \Y)$, 
such that $q^N \to q$ uniformly.
Since $q > 0$, Corollary~\ref{prop:JointLimit} 
yields $J_M^N(q^N) \to J_M(q)$, almost surely. 
Furthermore, due to the $\Gamma$-convergence of $(\iota_{Q^N})_N$ to $\iota_{Q}$,
it holds $\iota_{Q}(q) \leq \liminf_{N \to \infty} \iota_{Q^N}(q_N)$.
In total, we obtain
\begin{align}\label{proof:Gamma:1}
F(q) &= J_M(q) + \iota_{Q}(q) 
\leq \bigl(\lim_{N \to \infty} J_M^N(q_N) \bigr) 
+ \bigl(\liminf_{N \to \infty} \iota_{Q^N}(q_N)\bigr)
\nonumber
\\
&= \liminf_{N \to \infty} \big( J_M^N(q_N) + \iota_{Q^N}(q_N) \big)
= \liminf_{N \to \infty} F_N(q_N),
\end{align}
almost surely.
The $\Gamma$-convergence of $(\iota_{Q^N})_N$ to $\iota_{Q}$ also
ensures the existence of a sequence $(q_N)_N \subset C(\X \times \Y)$ 
with uniform limit $q$
such that $\limsup_{N \to \infty} \iota_{Q^N}(q_N) \leq \iota_Q(q)$.
Repeating the analogous steps to \eqref{proof:Gamma:1},
we obtain $\limsup_{N \to \infty} F_N(q_N) \leq F(q)$.
Thus, $(F_N)_N$ $\Gamma$-converges to $F$.

For compact $Q^N$, the functional
$F_N = J_M^N + \iota_{Q^N}$ admits a minimizers $q_N$ for all $N \in \N$ (and all $\omega \in \Omega$).
The fundamental theorem of $\Gamma$-convergence~\cite{Braides02}
readily ensures that, almost surely, 
any cluster point of $(q_N)$ minimizes $F = J_M + \iota_{Q}$, as desired.
\end{proof}

Finally, we show that despite the approximation of the sampling model that underlies the functional $J_M$, it can recover the true relation $\pi$ between the data points. To this end, we recall an auxiliary lemma.

\begin{lemma}
\label{lem:SptVsAe}
    Let $f,g: \X \to \R$ be continuous 
    and $\mu \in \mathcal{M}_+(\X)$. 
    Then it holds $f=g$ $\mu$-a.s.\ if and only if $f=g$ on $\spt(\mu)$.
\end{lemma}

\begin{proof}
    Let $f=g$ on $\spt(\mu)$.
    Since $A \subset \X \setminus \spt(\mu)$ 
    implies $\mu(A) = 0$,
    we readily obtain $f=g$ a.s..

    Let $f=g$ a.s..
    Assume, there exists $x \in \spt(\mu)$ such that $f(x) \neq g(x)$.
    The continuity of $f$ and $g$ yields the existence of $\delta > 0$
    such that this inequality 
    extends to the open $\delta$-ball $B_\delta(x)$ around $x$.
    In other words, 
    $f(\tilde{x}) \neq g(\tilde{x})$
    for all $\tilde{x} \in B_\delta(x)$.
    Since $x \in \spt(\mu)$, 
    it holds $\mu(B_\delta(x)) > 0$ which yields a contradicton.
\end{proof}

\begin{proposition}[Population minimizer]
\label{prop:PopMin}
Let $Q$ be a compact, convex subset of $\mathcal C(\X \times \Y)$.
Then $J_M$ has a minimizer over $Q$. 
Any two minimizers $q_1$ and $q_2$ are equal $\pi$-almost everywhere,
i.e.~the minimizer is unique in $L^1(\pi)$.
Furthermore, if $\pi = p\cdot \mu\otimes \nu$ for some $p \in Q$, then $p$ is a minimizer.
\end{proposition}

\begin{proof}
    Due to the lower semi-continuity of $J_M$ and compactness of $Q$, 
    existence of a minimizer is guaranteed by the Weierstrass theorem.
    Let $q_1,q_2 \in Q$ be two minimizers of $J_M$ over $Q$.
    We show that they are equal on $\spt \pi$.
    Due to linearity of the integral and strict convexity of $-\log$,
    we readily obtain $(t_M q_1) (\bm{x},y) = (t_M q_2) (\bm{x},y)$
    for $\pi \otimes \mu^{\otimes (M-1)}$-a.e.\ $(\bm{x},y)$.
    Due to $t_M q_1, t_M q_2 \in \mathcal{C}(\X^M \times \Y)$,
    we further obtain by Lemma \ref{lem:SptVsAe} that $t_M q_1 (\zb x,y)= t_M q_2(\zb x,y)$
    for all $(\bm{x},y) \in \spt(\pi \otimes \mu^{\otimes (M-1)})$.
    Let $(x,y) \in \spt(\pi)$. 
    In particular, this gives $x \in \spt (\mu)$,
    such that $(x,\dotsc,x,y) \in \spt(\pi \otimes \mu^{\otimes (M-1)})$
    and thus
    \begin{align*}
    q_1(x,y) 
    &= \frac{1}{M} \sum_{k=1}^M q_1(x,y) 
    = (t_M q_1) (x,\dotsc,x,y)
    \\
    &= (t_M q_2) (x,\dotsc,x,y)
    = \frac{1}{M} \sum_{k=1}^M q_2(x,y) 
    = q_2(x,y).
    \end{align*}
    Finally, let $\pi = p \cdot (\mu \otimes \nu)$ for some $p \in Q$.
    Proposition~\ref{prop:KL_form} ii) 
    gives that the set of solutions to $\min_{q \in Q} J_M(q)$
    coincides with the set of solutons of
    \[
    \min_{q \in Q} 
    \KL( 
    P_\# \mathcal{S}_M (p \cdot (\mu \otimes \nu))^{\otimes M} \vert 
    P_\# \mathcal{S}_M (q \cdot (\mu \otimes \nu))^{\otimes M}).
    \]
    By non-negativity of the $\KL$-divergence, 
    a solution is provided by $q = p$ as desired.
\end{proof}

\begin{remark}[Relation between the permutation and inference functionals]
\label{rem:Comparison}
Under the assumptions of 
Propositions~\ref{prop:1} and \ref{prop:KL_form}, 
the permutation and inference functional 
admit the respective forms
\begin{align*}
\mathcal{J}_M(q)  
& = 
\KL(\mathcal{S}_M \pi^{\otimes M} \vert 
\mathcal{S}_M (q \cdot \mu \otimes \nu)^{\otimes M}) 
+ \const, 
\\
J_M(q) 
& = 
\KL(P_\# \mathcal{S}_M \pi^{\otimes M} \vert
P_\# \mathcal{S}_M (q \cdot \mu \otimes \nu)^{\otimes M}) + \const.
\end{align*}
By \cite[Lem.~3.15]{LinHK2021}, it holds
\begin{equation}\label{eq:KL_estimate}
    \KL(T_\# \mu, T_\# \tilde{\mu})
    \leq
    \KL(\mu,\tilde{\mu})
\end{equation}
for any measurable $T:\X \to \Y$.
Hence we obtain
\[
\KL(P_\# \mathcal{S}_M \pi^{\otimes M} \vert
P_\# \mathcal{S}_M (q \cdot \mu \otimes \nu)^{\otimes M})
\leq 
\KL(\mathcal{S}_M \pi^{\otimes M} \vert 
\mathcal{S}_M (q \cdot \mu \otimes \nu)^{\otimes M}).
\]
In view of the above Proposition~\ref{prop:PopMin}, 
if $p \in Q$, $p$ is the essentially unique minimizer
for both functionals $\mathcal{J}_M$ and $J_M$.
In this sense, 
the approximation underlying the latter 
does not introduce a systematic bias. 
But the minimum of the latter may be less pronounced, 
i.e.\ we may lose some information 
contained in the data relative to $\mathcal{J}_M$.
We study the information 
that can be recovered by $J_M^N$ numerically
for larger $M$ in Section \ref{sec:numerics}.
\end{remark}

%------------------------------------------------------------------------------------
\section{Non-parametric estimation with entropic transport kernels}
\label{sec:Kernels}
Above, we propose an approximate empirical maximum likelihood functional for inferring the density of the measure $\pi$ and its population limit.
Now we introduce a suitable non-parametric class of hypothesis densities 
$Q \subset \mathcal C_+(\X \times \Y)$.
To control the bias-variance trade-off, when doing inference on finitely many samples, 
we need to be able to control the complexity of $Q$ and adjust it to the number of available samples $N$.
In kernel density estimation, this can be done by controlling the width of the kernels. 
We use entropic optimal transport to construct such kernels. In particular, this naturally allows to incorporate disintegration constraints.
The entropic regularization parameter $\varepsilon$ controls the width of the kernels, with the width given as approximately $\sqrt{\varepsilon}$.
Some necessary background on entropic optimal transport is collected in Subsection \ref{sec:KernelsBackground}.
The class of kernels $Q$ is constructed in Subsection \ref{sec:Kernel}. 
In particular, we also consider the (common) case when $\mu$ and $\nu$ are unknown and show how $Q$ can be approximated from empirical estimates in an asymptotically consistent way. 
This approximation is finite-dimensional and therefore also amenable for numerical methods.
In preparation for the numerical minimization in Section \ref{sec:Algorithm}, 
we give an explicit form of the discrete inference functional $J_M^N$ in Subsection \ref{sec:KernelsDisc}.

%--------------------------------------------------------------------------
\subsection{Entropic optimal transport} \label{sec:KernelsBackground}
%--------------------------------------------------------------------------
The following can be found in serveral works, e.g., in \cite{NS2021},
see also \cite{BoSch2020,FSVATP2018,santambrogio2015optimal}. 

Let $(\X,\text{dist})$ be a compact metric space and $c \coloneqq \text{dist}^2$
a cost function.
Then, for $\varepsilon >0$, the \emph{entropic optimal transport}
between $\mu, \tilde \mu \in \mathcal P(\X)$
is given by
\begin{align}\label{eq:EntropicPrimal}
\text{OT}_\varepsilon \coloneqq \inf_{\pi \in \Pi(\mu,\tilde \mu) }
\left\{ \int_{\X^2} c(x,\tilde x) \diff \pi(x,\tilde x) 
+ \varepsilon \, \KL(\pi|\mu \otimes \tilde \mu)
\right\} ,
\end{align}
where 
$
\Pi(\mu,\tilde \mu) \assign \{ \pi \in \mathcal P(\X \times \X):
\, (P_1)_\# \pi= \mu, (P_2)_\#  \pi = \tilde \mu\}
$ 
and $P_i(x_1,x_2) \coloneqq x_i$, $i=1,2$.
This minimization problem has a unique mimimizer $\hat \pi$.
Problem \eqref{eq:EntropicPrimal} can be reformulated in a dual form as
\begin{align} \label{eq:EntropicDual}
\text{OT}_\varepsilon = 
\sup_{ \phi,\psi \in \mathcal C(\X)} 
\left\{ \int_{\X} \phi \, \diff \mu 
+ 
\int_{\X} \psi \, \diff \tilde \mu 
- 
\varepsilon \int_{\X^2} \left[\exp([\phi \oplus \psi-c]/\varepsilon)-1\right]\,\diff \mu \otimes \tilde \mu \right\}.
\end{align}
Here $\phi \oplus \psi$ denotes the function $\X \times \X \ni (x,y) \mapsto \phi(x)+\psi(y)$.
Indeed, there exist optimal maximizers  $\hat \phi, \hat \psi \in \mathcal C(\X)$ 
and they are (up to an additive constant) unique on $\text{spt} \, \mu$ 
and $\text{spt} \, \tilde \mu$, respectively. 
For any pair of optimal maximizers $(\hat \phi, \hat \psi)$ of \eqref{eq:EntropicDual}, it holds
\begin{equation} \label{eq:EntropicPD}
\hat \pi = \exp([(\hat \phi \oplus \hat \psi) - c]/\varepsilon) \cdot \mu \otimes \tilde \mu
\end{equation}
and
\begin{equation}\label{eq:Sinkhorn}
\begin{aligned}
\hat \phi(x) & = -\varepsilon \log\left(\int \exp([\hat \psi(y)-c(x,y)]/\varepsilon)\,\diff \mu(y) \right), \\
\hat \psi(\tilde x) & = -\varepsilon \log\left(\int \exp([\hat \phi(y)-c(y,\tilde x)]/\varepsilon)\,\diff \tilde \mu(y) \right)
\end{aligned}
\end{equation}
for $\mu$-almost all $x$ and $\tilde \mu$-almost all $\tilde x$.
There are dual maximizers that satisfy \eqref{eq:Sinkhorn} for all $x,\tilde x \in \X$:
this can be seen since \eqref{eq:Sinkhorn} can be evaluated for any $x, \tilde x \in \X$ 
which extends the dual maximizers in a continuous way to the full space $\X$.

The following properties of entropic optimal transport will underlie our construction of   hypothesis spaces.

\begin{proposition}[Entropic transport kernels]\label{prop:EntropicKernels}
Let $(\hat \phi, \hat \psi)$ be maximizers of \eqref{eq:EntropicDual} 
that satisfy \eqref{eq:Sinkhorn} for all $x,\tilde x \in \X$.
We call 
$$
k^\varepsilon \assign \exp([\hat \phi \oplus \hat \psi-c]/\varepsilon) \in \mathcal C(\X \times \Y)
$$
the {\rm entropic transport kernel associated with} $\mu$ and $\tilde \mu$.
Then the following relations holds true:
\begin{enumerate}[(i)]
	\item $k^\varepsilon$ is unique.
	\item $k^\varepsilon > 0$.
   \item $k^\varepsilon$ is bounded from above and Lipschitz continuous.
	\item 
 $\int_{\X} k^\varepsilon (x,\tilde x )\,\diff \tilde \mu (\tilde x)= 1$ 
 for all $x \in \X$ and 
 $\int_{\X} k^\varepsilon (x,\tilde x)\,\diff \mu(x)=1$ for all $\tilde x \in \X$.
	\item 
 Let $(\mu^N)_N$, $(\tilde \mu^N)_N$ 
 be sequences of probability measures which converge weak* to $\mu$ and $\tilde \mu$ respectively, and let $(k^\varepsilon_N)_N$ be the sequence of associated transport kernels. Then $k^\varepsilon_N \to k$ as $N \to \infty$ uniformly in $\mathcal C(\X \times \X)$.
\end{enumerate}
\end{proposition}

\begin{proof}
i)  
Since the dual maximizers $\hat \phi$ and $\hat \psi$ are unique $\mu$ and $\tilde \mu$-almost everywhere up to an additive constant,
the extension of $\hat \phi$ and $\hat \psi$ to all of $\X$ via \eqref{eq:Sinkhorn} only depends on the values of the other function 
$\tilde \mu$-a.e.\ or $\sigma$-a.e., respectively. 
Hence, these extensions are unique up to an additive constant, 
which makes $\hat \phi \oplus \hat \psi$ unique for all dual maximizers that solve \eqref{eq:Sinkhorn} on the full space. 
Therefore, $k^\varepsilon$ is unique and well-defined.
\\
ii)
Since $c$ and $\hat \phi \oplus \hat \psi$ are continuous on a compact domain, $k^\varepsilon$ is bounded away from zero.
\\
iii)
By compactness of $\X$, the cost function $c=\text{dist}^2$ is Lipschitz continuous. 
The solutions to \eqref{eq:Sinkhorn} inherit the Lipschitz constant of $c$ and $\hat \phi \oplus \hat \psi$ can be shown to be bounded (see \cite[Proposition 1.11]{santambrogio2015optimal} for the same arguments for unregularized transport). Therefore, $\hat \phi \oplus \hat \psi-c$ is Lipschitz continuous and bounded from above, and so is $k^\varepsilon$.
\\
iv)
This part follows from $\pi=k^\varepsilon \cdot (\mu \otimes \tilde \mu) \in \Pi(\mu,\tilde \mu)$.
\\
v)
Let $\hat \pi^N$ be the unique minimizer of $\text{OT}_\varepsilon$ for $(\mu^N,\tilde \mu^N)$ and 
let $(\hat \phi^N, \hat \psi^N)$ be dual maximizers 
that solve \eqref{eq:Sinkhorn} for all $x,\tilde x \in \X$.
Since $(\hat \pi^N)_N$ is a sequence of probability measures on a compact domain, 
the sequence is weak* pre-compact, so that it has a convergent subsequence. 
Now any cluster point $\pi$ can be shown to satisfy $\pi \in \Pi(\mu,\tilde \mu)$ and 
by the weak* lower semi-continuity of $\KL$, we have
$$
\liminf_{N \to \infty} \int_{\X \times \X} c\,\diff \pi^N + \varepsilon \KL(\pi^N|\mu^N \otimes \tau^N) \geq \int_{\X \times \X} c\,\diff \pi + \varepsilon \KL(\pi|\sigma \otimes \tau).
$$
Arguing as in point iii), the $(\hat \phi^N,\hat \psi^N)$ are equi-Lipschitz, 
and by applying suitable constant shifts, they can be shown to be equi-bounded (see again \cite[Proposition 1.11]{santambrogio2015optimal}). 
Hence, by the theorem of Ascoli--Arzela, there is some cluster point $(\phi,\psi)$ of
$( \hat \phi^N, \hat \psi^N)$ with respect to uniform convergence 
for which the dual objective values of $(\hat \phi^N,\hat \psi^N)$ converge to that of $(\phi,\psi)$. 
Hence, the cluster points $\pi$ and $(\phi,\psi)$ must be primal and dual optimal for the limit problems and therefore $k^\varepsilon_N \to k^\varepsilon$ holds true as $N \to \infty$.
\end{proof}

%----------------------------------------------------------------------------
\subsection{Non-parametric class of hypothesis densities}
\label{sec:Kernel}
%----------------------------------------------------------------------------
In the following, we  construct the hypothesis densities $Q$.

\begin{definition}[Hypothesis space] \label{prop:QConstruction}
For $\mu, \tilde{\mu} \in \mathcal P(\X)$, $\nu, \tilde{\nu} \in \mathcal P(\Y)$, let $k^\varepsilon_{\mu\tilde{\mu}} \in \mathcal C(\X \times \X)$ 
and 
$k^\varepsilon_{\nu\tilde{\nu}} \in \mathcal C(\Y \times \Y)$ be the entropic OT kernels between $\mu,\tilde{\mu}$ and $\nu,\tilde{\nu}$, respectively and 
$K = K_{\mu,\tilde \mu,\nu, \tilde \nu}^\varepsilon \coloneqq k^\varepsilon_{\mu\tilde{\mu}} k^\varepsilon_{\nu\tilde{\nu}}
\in \mathcal C\big( (\X \times \X)\times (\Y \times \Y) \big)
$.
We call the linear mapping
\begin{align*}
E_K: \mathcal M(\X \times \Y) \to \mathcal C(\X \times \Y), \quad
	\xi \mapsto k^\varepsilon_{\mu\tilde{\mu}}.\xi.k^\varepsilon_{\nu\tilde{\nu}} 
 :=
	\int_{\X \times \Y} 
 k^\varepsilon_{\mu\tilde{\mu}}(x,\tilde x)\,k^\varepsilon_{\nu\tilde{\nu}}(y,\tilde y)
 \,\diff \xi (\tilde x, \tilde y).
\end{align*}
{\rm entropic kernel mean embedding}. Then we propose as  {\rm hypothesis space of densities} 
\begin{align} \label{eq:Q}
	Q \assign \{ E_K(y) = k^\varepsilon_{\mu\tilde{\mu}}.\xi.k^\varepsilon_{\nu\tilde{\nu}} | \xi \in \Xi (\X \times \Y) \},
\end{align}
for some $\Xi \subseteq  \mathcal P(\X \times \Y)$.
\end{definition}
\medskip

In our applications, we will restrict ourselves to 
$\Xi  \in \big\{  \mathcal P(\X \times \Y), \Pi(\tilde \mu) \big\}$, where 
\begin{equation} \label{hyp1}
\Pi(\tilde \mu) \assign \{\xi \in \mathcal P(\X \times \Y) : (P_1)_\# \xi = \tilde \mu\}.
\end{equation}
Note that both sets $\mathcal P(\X \times \Y)$ 
and $\Pi(\tilde \mu)$
are weak* compact.

\begin{remark}
Indeed, the entropic kernel mean embedding maps into Lipschitz continuous functions in $\mathcal C(\X \times \Y)$ by the following argument:
Let $(x_1,y_1)$ and $(x_2,y_2)$ in $\X \times \Y$. Then
\begin{align*}
  & |k^\varepsilon_{\mu\tilde{\mu}}.\xi.k^\varepsilon_{\nu\tilde{\nu}}(x_1,y_1)-k^\varepsilon_{\mu\tilde{\mu}}.\xi.k^\varepsilon_{\nu\tilde{\nu}}(x_2,y_2)| \\
  \leq & \left|
\int_{\X \times \Y} 
 \big(k^\varepsilon_{\mu\tilde{\mu}}(x_1,\tilde x)\,k^\varepsilon_{\nu\tilde{\nu}}(y_1,\tilde y)-k^\varepsilon_{\mu\tilde{\mu}}(x_2,\tilde x)\,k^\varepsilon_{\nu\tilde{\nu}}(y_2,\tilde y)\big)
 \,\diff \xi (\tilde x,\tilde  y) 
 \right| \\
 \leq & C \cdot \big( \textnormal{dist}_\X(x_1,x_2) + \textnormal{dist}_\Y(y_1,y_2) \big)
\end{align*}
where $C$ depends on the Lipschitz constants and upper bounds on $k^\varepsilon_{\mu\tilde{\mu}}$ and $k^\varepsilon_{\nu\tilde{\nu}}$ as implied by Proposition \ref{prop:EntropicKernels}iii).
\end{remark}

The following remark shows the relation to kernel mean embeddings of measures.

\begin{remark}
If $\mu = \tilde \mu$,  then we can choose the dual functions in 
$\text{OT}_\varepsilon(\mu,\tilde \mu)$ that $\phi_\mu = \psi_\mu$,
and similarly for $\nu = \tilde \nu$.
In this case, the kernel becomes 
\begin{align}
K\big((x,y),(\tilde x , \tilde y) \big)
=\exp\Big( \frac{\phi_\mu(x) + \phi_\nu(y)}{\varepsilon} \Big)
\exp\Big(\frac{-c(x,\tilde x) - c(y,\tilde y)}{\varepsilon} \Big)
\exp\Big( \frac{\phi_\mu(\tilde{x}) + \phi_\nu(\tilde{y})}{\varepsilon} \Big).
\end{align}
This is a symmetric kernel, which is moreover positive definite.
Then it is well-known that the kernel mean embedding $E_K$  maps into the 
reproducing kernel Hilbert space $\mathcal H_K$ with reproducing kernel $K$.
Furthermore, for so-called characteristic kernels $K$, the map $E_K$ is injective
and surjective onto $\mathcal H_K$ if and only if $\X \times \Y$ is finite, see  \cite{steinwart2021strictly}.
For more information, we refer to \cite{SGS2018kernel, steinwart2008support}.
\end{remark}

By the next  proposition, we will see that any $q \in Q$  is a probability density with respect to $\mu \otimes \nu$.
Furthermore, in the case 
$\Xi = \Pi(\tilde \mu)$, the density
$q(x,y) \in Q$ can be interpreted as conditional density of $y$ given the location $x$ of a particle pair $(x,y)$. 
As discussed in the introduction, this is particularly useful for the estimation of transition probabilities and transfer operators in dynamical systems.

\begin{proposition}[Mass preservation]\label{prop:Markov}
Let the assumptions of Definition \ref{prop:QConstruction}
be fulfilled.
Then it holds 
for any $q \in Q$ that
$q \cdot (\mu \otimes \nu) \in \mathcal P(\X \times \Y)$. 
If moreover $\Xi = \Pi(\tilde \mu)$, then we have for any $x \in \X$ that 
$q(x,\cdot) \cdot \nu \in \mathcal P(\Y)$.
\end{proposition}

\begin{proof}
Assume that $q = E_K(\xi)$, where $\xi \in \mathcal P(X \times Y)$. Clearly $q$ is non-negative and by Proposition \ref{prop:EntropicKernels}iv), we obtain
\begin{align*}
	 \int_{\X \times \Y} q(x,y)\,\diff \mu(x)\,\diff \nu(y) = \int_{\X \times \Y} k^\varepsilon_{\mu\tilde{\mu}}(x, \tilde x)k^\varepsilon_{\nu\tilde{\nu}}(y,\tilde y) \, \diff \xi( \tilde x, \tilde y) \diff \mu(x) \diff \nu(y) = 1.
\end{align*}
If in addition $(P_1)_\# \xi =\tilde{\mu}$, we conclude for $x \in \X$ that
\begin{align*}
    \int_{\Y} q(x,y) \, \diff \nu(y) 
    &= \int_{\Y} \int_{\X \times \Y} 
    k^\varepsilon_{\mu\tilde{\mu}}(x,\tilde x )k^\varepsilon_{\nu\tilde{\nu}}(y,\tilde y)
    \, \diff \xi(\tilde x,\tilde y) \diff \nu(y)\\ 
    &=\int_{\X} k^\varepsilon_{\mu\tilde{\mu}}(x,\tilde x) \, \diff((P_1)_\#\xi)(\tilde x) = 1.   
    \qedhere
\end{align*}
\end{proof}

\begin{lemma}
The hypothesis space $Q$  constructed in \eqref{eq:Q} is a compact subset of $\mathcal C(\X \times \Y)$ 
if $\Xi$ is weak* compact.
It is a polyhedral subset of $\mathcal C(\X \times \Y)$, if $\Xi$ is polyhedral, and 
it is finite-dimensional, if $\Xi$ is finite-dimensional.
\end{lemma}

\begin{proof}
Since the kernels $k^\varepsilon_{\mu\tilde{\mu}}$ and 
$k^\varepsilon_{\nu\tilde{\nu}}$ are continuous, the map $E_K: \xi \mapsto k^\varepsilon_{\mu\tilde{\mu}}.\xi.k^\varepsilon_{\nu\tilde{\nu}}$ is continuous between weak* topology on $\mathcal M(X \times Y)$ and $\mathcal C(X \times Y)$.
Compactness of $Q=E_K(\Xi)$ then follows from compactness of $\Xi$.

The other assertions follows by linearity of the entropic kernel mean embedding.
\end{proof}

 The next proposition shows how the class $Q$ in \eqref{eq:Q} can be approximated, for instance by finite-dimensional sets and/or when the true measures $\mu$ and $\nu$ are unknown and only empirical approximations are available.

\begin{proposition}[Empirical approximation of hypothesis class]
\label{prop:KernelsQNConvergence}
Let $(\mu^N)_N$, $(\tilde{\mu}^N)_N \subset \mathcal{P}(\X)$ 
and $(\nu^N)_N$, $(\tilde{\nu}^N)_N \subset \mathcal{P}(\Y)$ 
be sequences of probability measures
with weak* limits
$\mu, \tilde{\mu} \in \mathcal P(\X)$ and
$\nu, \tilde{\nu} \in \mathcal P(\Y)$, respectively.
Let $\Xi \subset \mathcal P(\X \times \Y)$ 
and let $(\Xi^N)_N$ be a sequence of 
subsets of $\mathcal P(\X \times \Y)$, 
such that the $\Gamma$-limit 
with respect to the weak* topology 
of the sequence of indicator functions $(\iota_{\Xi^N})_N$ 
is $\iota_{\Xi}$.
For fixed $\varepsilon>0$, 
let $Q$ and $Q^N$ be the respective set of densities 
as in \eqref{eq:Q}, i.e.,
\begin{align*}
Q  \assign 
\{
k^\varepsilon_{\mu\tilde{\mu}}.\xi.k^\varepsilon_{\nu\tilde{\nu}} 
: \xi \in \Xi
\} 
\quad \text{and} \quad
Q^N  \assign 
\{ 
k^\varepsilon_{\mu_N\tilde{\mu}_N}.\xi.k^\varepsilon_{\nu_N\tilde{\nu}_N} 
: \xi \in \Xi^N 
\}.
\end{align*}
Then the $\Gamma$-limit of the sequence $(\iota_{Q^N})_N$ is $\iota_{Q}$.
\end{proposition}

\begin{proof}
By the assumptions and Propositon~\ref{prop:EntropicKernels}, 
we have that $(k^\varepsilon_{\mu_N\tilde{\mu}_N})_N$ 
and $(k^\varepsilon_{\nu_N\tilde{\nu}_N})_N$ 
converge uniformly to 
$k^\varepsilon_{\mu\tilde{\mu}}$ 
and $k^\varepsilon_{\nu\tilde{\nu}}$.
We check that the sets $(Q^N)_N$ and $Q$ 
fulfill i) and ii) of Remark~\ref{rem:Gamma_convergence_for_iota}
with respect to uniform convergence.
\\[1ex]
i) Let $q \not \in Q$ and 
$(q_N)_N \subset \mathcal{C}(\X \times \Y)$ 
so that $q_N \to q$, $N \to \infty$.
We show that there exists $\bar{N} \in \N$
so that $q_N \not\in Q_N$ for all $N \geq \bar{N}$
by a contradiction.
Hence, assume that $q_N \in Q_N$ 
for infinitely many $N \in \N$.
Up to picking a subsequence, 
we may assume that $q_N \in Q_N$ for all $N \in \N$.
By definition, there exist $\xi_N \in \Xi^N$,
so that $E_K(\xi_N) = q_N$, $N \in \N$.
Up to picking a further subsequence,
there exists $\xi \in \mathcal{P}(\X \times \Y)$,
so that $\xi_N \rightweaks \xi$.
Due to the $\Gamma$-convergence relation of 
$(\Xi^N)_N$ and $\Xi$,
Remark~\ref{rem:Gamma_convergence_for_iota}
ensures that $\xi \in \Xi$.
Finally, by continuity of $E_K$, we obtain
\[
E_K(\xi) 
= \lim_{N \to \infty} E_K(\xi_N) 
= \lim_{N \to \infty} q_N 
= q,
\]
which yields the desired contradiction 
to the assumption $q \not \in Q$.
\\[1ex]
ii) Let $q \in Q$. 
We construct $(q_N)_N$ with $q_N \to q$,
such that $q_N \in Q_N$, for all except finitely many $N \in \N$.
By definition, there exists $\xi \in \Xi$, 
so that $E_K(\xi) = q$.
Using again the $\Gamma$-convergence assumption of 
$(\Xi^N)_N$ and $\Xi$
together with Remark~\ref{rem:Gamma_convergence_for_iota}
we obtain the existence of 
$(\xi_N)_N \subset \mathcal{P}(\X \times \Y)$
so that $\xi_N \in \Xi^N$ 
for all but finitely many $N \in \N$.
Then, by continuity of $E_K$, it holds
\[
q_N \coloneqq E_K(\xi_N) \to E_K(\xi) = q.
\]
Additionally, by construction, $q_N \in Q_N$ for all but finitely many $N \in \N$ which concludes the proof.
\end{proof}

\begin{proposition} \label{prop:RN}
Let $(\tilde{X}^N)_N$ and $(\tilde{Y}^N)_N$ be sequences of increasing subsets of $\X$ and $\Y$ such that $\bigcup_N \tilde{X}^N$ and $\bigcup_N \tilde{Y}^N$ are dense in $\X$ and $\Y$ respectively.
Let $(\tilde{\mu}^N)_N$ be a sequence of measures with $\spt \tilde{\mu}^N \subset \tilde{X}^N$, converging weak* to $\tilde{\mu}$.
Let discrete approximations of $\Xi_1 \assign \mathcal P(\X \times \Y)$ and $\Xi_2 \assign \Pi(\tilde \mu)$ 
be given by
\begin{align} \label{eq:RN}
\Xi^N_{1}  \assign \mathcal P(\tilde{X}^N \times \tilde{Y}^N)
\quad \text{or} \quad
\Xi^N_{2}  \assign \{\xi \in \mathcal P(\tilde{X}^N \times \tilde{Y}^N) : (P_{1})_\# \xi = \tilde{\mu}^N\}.
\end{align}
Then,  the $\Gamma$-limit of the sequence of functions $(\iota_{\Xi^N_i})_N$ 
is given by $\iota_{\Xi_i}$,  $i=1,2$.
\end{proposition}

\begin{proof}[Sketch of proof]
Since $\bigcup_N \tilde{X}^N$ is dense in $\X$ and the latter is compact, for each $\delta>0$ there is some $N$ such that for all $x \in \X$ one has $B_\delta(x) \cap \tilde{X}^N \neq \emptyset$ where $B_\delta(x)$ denotes the open ball of radius $\delta$ in $\X$ centered at $x$. Therefore, for each $\mu \in \mathcal{P}(\X)$ there is some $\mu^N \in \mathcal{P}(\tilde{X}^N)$ such that the Wasserstein distance (for any $p \in [1,\infty]$) between $\mu$ and $\mu^N$ is less than $\delta$. This implies that $\bigcup_N \mathcal{P}(\tilde{X}^N)$ is a dense subset of $\mathcal{P}(\X)$ with respect to the Wasserstein distance.
By the same argument on the product space $\X \times \Y$, approximated by the product sets $\tilde{X}^N \times \tilde{Y}^N$, $\bigcup_N \Xi^N_1$ is a dense subset of $\Xi_1$ with respect to the Wasserstein distance on $\X \times \Y$. 
We show conditions i) and ii) 
of Remark~\ref{rem:Gamma_convergence_for_iota}
for the case $i=1$.
First, assume 
$\xi \not\in \Xi_1 = \mathcal{P}(\X \times \Y)$ 
and let
$\xi_N \rightweaks \xi$ 
for some $(\xi_N)_N \subset \mathcal{M}(\X \times \Y)$. 
Since $\mathcal{P}(\X \times \Y)$ is closed,
it follows 
$\xi_N \not\in \mathcal{P}(\X \times \Y)
\supset \Xi_1^N$ for all but finitely many $N \in \N$.
Condition ii) of Remark~\ref{rem:Gamma_convergence_for_iota}
follows directly since $(\Xi_1^N)_N$ is dense in $\Xi_1$.
This establishes the $\Gamma$-limit in the case $i=1$. 
The fact that $\tilde{\mu}^N \rightweaks \tilde{\mu}$ allows to deal with the constraint for $i=2$.
\end{proof}

When $\tilde{X}^N$ and $\tilde{Y}^N$ are finite, 
then
$\Xi^N_{1}$ and $\Xi^N_{2}$ are finite-dimensional.
So with Proposition \ref{prop:KernelsQNConvergence}, Corollary \ref{cor:Gamma} and Proposition \ref{prop:RN} we can approximately infer a minimizer of $J_M$ over $Q$. By Proposition \ref{prop:PopMin}, we can intuitively expect this minimizer to be `close' to $p$, if $p$ is `close' to $Q$.

\begin{remark}[Estimation of transfer operators and spectral analysis]
Let $q \in \mathcal C(\X \times \Y)$ such that $q \cdot (\mu \otimes \nu) \in \Pi(\mu,\nu)$, i.e.~$q$ represents a set of transition probability densities.
As discussed in Section \ref{sec:Intro}, $q$ induces a transfer operator $\mathcal T : L^1(\mu) \to L^1(\nu)$ via $(\mathcal Tu)(y)\assign\int_{\X} q(x,y) u(x)\,\diff \mu(x)$.
The operator $\mathcal T$ maps probability densities in $L^1(\mu)$ to probability densities in $L^1(\nu)$.
Since $q$ is bounded, $\mathcal T$ is in fact an operator from $L^p(\mu) \to L^p(\nu)$ for any $p \in [1,\infty]$. In particular, for $p=2$, spectral analysis of $\mathcal T$ can yield an informative low-dimensional description of the macroscopic properties of the underlying dynamics.

Let now $(\mu^N)_N$ and $(\nu^N)_N$ be sequences in $\mathcal{P}(\X)$ and $\mathcal{P}(\Y)$ that converge weak* to $\mu$ and $\nu$, respectively. Let $(q^N)_N$ be a sequence in $\mathcal C(\X \times \Y)$, converging uniformly to $q$ such that $(q^N)_N$ represents a set of transition probability densities from $\mu^N$ to $\nu^N$.
Then $q^N$ induces a transfer operator $\mathcal T^N : L^p(\mu^N) \to L^p(\nu^N)$.
It was shown in \cite[Sections 4.5 to 4.7]{EntropicTransfer22} that a suitable extension of $\mathcal T^N$ to $L^p(\mu) \to L^p(\nu)$ converges to $\mathcal T$ in the Hilbert--Schmidt norm, and thus spectral analysis of $\mathcal T^N$ can be used for an approximate spectral analysis of $\mathcal T$.
When $\mu^N$ and $\nu^N$ have finite support, $\mathcal T^N$ can be explicitly represented and obtained numerically as a finite matrix and its eigen- or singular vectors can be extracted. We give a numerical example for transfer operator analysis in Section \ref{sec:gyre}.
\end{remark}

%---------------------------------------------------------------------------
\subsection{Explicit discrete functional}\label{sec:KernelsDisc}
%---------------------------------------------------------------------------
For $N \in \N$, samples  $(\zb x^i,\zb y^i) = ((x_{j}^i,y_{j}^i)_{j=1}^M)_{i=1}^N$  from $\zb Z$ in \eqref{Z}, we set 
$$
\mathcal X^N \assign \{x_{j}^i : i \in [N],\,j \in [M]\} 
\quad  \text{and} \quad
\mathcal Y^N \assign \{y_{j}^i : i \in [N],\,j \in [M]\}.
$$ 
We associate with the $i$-th sample the empirical measures 
$$
\mu_i^N \assign \frac{1}{M} \sum_{j=1}^M \delta_{x_{j}^i} 
\quad  \text{and} \quad
\nu_i^N \assign \frac{1}{M} \sum_{j=1}^M \delta_{y_{j}^i},
$$ 
and set finally
\begin{equation} \label{mue}
\mu^N\assign \frac{1}{N}\sum_{i=1}^N \mu_i^N
\quad  \text{and} \quad
\nu^N\assign \frac{1}{N}\sum_{i=1}^N \nu_i^N.
\end{equation}
In addition, we choose empirical probability measures 
$\tilde{\mu},\tilde{\nu}$ 
concentrated on some finite subsets 
$$
\tilde{\mathcal X} \assign \{ \tilde x_k: k \in [K] \}\subset \X 
\quad  \text{and} \quad
\tilde{\mathcal Y} \assign \{ \tilde y_l: l \in [L] \} \subset \Y,
$$
i.e.,
$$
\tilde \mu 
\assign  \sum_{k=1}^K \tilde \mu_k \delta_{\tilde x_{k}} 
\quad  \text{and} \quad
\tilde \nu \assign \sum_{l=1}^L \tilde \nu_l \delta_{\tilde y_{l}},
$$
where the $(\tilde \mu_k)_k \in \triangle_K$ and $(\tilde \nu_l)_l  \in \triangle_L$.
We want to solve
\begin{equation}\label{eq:DiscreteProblemPre}
\min_{\xi\in \mathcal P(\tilde{\mathcal X} \times \tilde{\mathcal Y})} J^N_M(k^\varepsilon_{\mu^N\tilde{\mu}}.\xi.k^\varepsilon_{\nu^N\tilde{\nu}}),
\end{equation}
or the constrained problem
\begin{equation}\label{eq:DiscreteProblemPre1}
\min_{\xi\in \mathcal P(\tilde{\mathcal X} \times \tilde{\mathcal Y})} J^N_M(k^\varepsilon_{\mu^N\tilde{\mu}}.\xi.k^\varepsilon_{\nu^N\tilde{\nu}}) \quad \text{subject to} \quad (P_1)_\# \xi = \tilde \mu.
\end{equation}
Note that in for our discrete setting, the measures 
$\mathcal P(\tilde{\mathcal X} \times \tilde{\mathcal Y})$ 
have the form
$$
\xi \assign \sum_{k,l=1}^{K,L} \xi_{k,l} \delta_{\tilde x_k,\tilde y_l}, \quad 
\sum_{k,l=1}^{K,L} \xi_{k,l} = 1, \; \xi_{k,l} \ge 0
$$
and the function 
$q \assign k_{\mu^N\tilde{\mu}}.\xi.k^\varepsilon_{\nu^N\tilde{\nu}}
\in \mathcal X^N \times \mathcal Y^N$ 
is given by
\begin{equation}\label{eq:q}
q(x^n_m, y^{i}_{j}) \assign \sum_{k,l=1}^{K,L}
k_{\mu^N\tilde{\mu}} (x^n_m,\tilde x_k) k^\varepsilon_{\nu^N\tilde{\nu}}(y^{i}_{j}, \tilde y_l) \xi_{k,l}.
\end{equation}
An illustration of the density 
$q$ is given in Figure~\ref{fig:k_rho_k}.
\begin{figure}
    \centering
    \includegraphics[width = 0.5\linewidth]{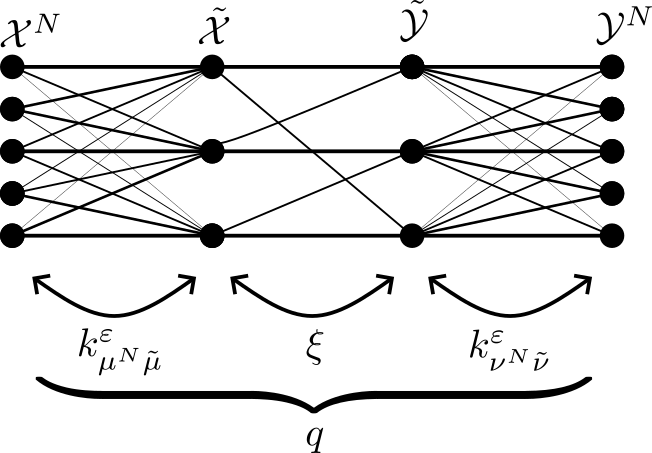}
    \caption{
    Illustration of the transport density 
    $q = k^\varepsilon_{\mu^N\tilde{\mu}}.\xi.k^\varepsilon_{\nu^N\tilde{\nu}}$.
    }
    \label{fig:k_rho_k}
\end{figure}

The following proposition provides an explicit expression for the above problem
in terms of the KL divergence, that can then be tackled with the algorithms in Section \ref{sec:Algorithm}.

\begin{proposition}\label{prop:DiscreteProblem}
For any $\xi\in\mathcal P(\tilde{\mathcal X} \times \tilde{\mathcal Y})$, it holds
\begin{align}
J^N_M(k^\varepsilon_{\mu^N\tilde{\mu}^N}.\xi.k^\varepsilon_{\nu^N\tilde{\nu}^N})
&=
\frac1N \sum_{n=1}^N \KL(\nu_n^N\mid (A^\varepsilon\xi)_n) + \text{\rm const}\\
&=
\KL(\tfrac{1}{N} \oplus_{n=1}^N \nu_n^N\mid A^\varepsilon\xi) + \text{\rm const}
\end{align}
with a constant independent of $\xi$.
Here $A^\varepsilon$ is the linear operator
$A^\varepsilon:{\mathcal P}(\tilde {\mathcal X} \times \tilde {\mathcal Y}) \to \mathcal P(\mathcal Y^N)$,
$\xi  \mapsto \frac{1}{N} \bigoplus_{n=1}^N (A^\varepsilon\xi)_n$ given with 
$q= k_{\mu^N\tilde{\mu}}.\xi.k^\varepsilon_{\nu^N\tilde{\nu}}$  by
\begin{align*} 
    (A^\varepsilon\xi)_n  
    &\coloneqq \left(\int_{\X} [k^\varepsilon_{\mu^N\tilde{\mu}}.\xi.k^\varepsilon_{\nu^N\tilde{\nu}}](x,\cdot)\,\diff \mu_n^N(x) \right)\cdot\nu^N\\
    &=
    \int_{\X} q(x,\cdot) \ \diff\mu^N_n(x)  \cdot \nu^N
    =
    \frac{1}{N M^2} \sum_{i=1}^N \sum_{j=1}^M \sum_{m=1}^M
    q(x^n_m,y^{i}_{j})  \, \delta_{y^{i}_{j}}.
    \end{align*}
\end{proposition}

\begin{proof}
Since $\xi\in\mathcal P(\tilde{\mathcal X}\times \tilde{\mathcal Y})$,
we can apply Proposition \ref{prop:EntropicKernels}iv) to get
\[
    \int_{\X \times \Y} q \ \diff(\mu^N\otimes\nu^N) 
    = 
    \int_{\X^2\times \Y^2}k^\varepsilon_{\mu^N\tilde{\mu}}(x,\tilde x)k^\varepsilon_{\nu^N\tilde{\nu}}(y,\tilde y)\diff \xi(\tilde x, \tilde y) \, \diff\mu^N(x) \, \diff\nu^N(y) = 1.
\]
Thus, $A^\varepsilon\xi \in \mathcal P(\mathcal Y^N)$ and using the above relation
with 
$$
\nu^N_n = \frac1M  \sum _{j=1}^M \delta_{y^n_j} + \sum_{i\not = n}^N  \sum _{j=1}^M 0 \, \delta_{y^i_j}
$$
we obtain 
\begin{align*}
  &\sum_{n=1}^N \KL\Big(\nu^N_n \mid  \int_{\X} q(x,\cdot) \ \diff\mu^N_n(x)  \cdot \nu^N\Big) 
   =  
   \sum_{n=1}^N \int_{\Y} 
   \log\Big(
   \frac{\diff \nu^N_n}{ 
   \diff \big( \int_{\X} q(x,\cdot) \ \diff\mu^N_n(x)  \cdot \nu^N \big)}  \Big) \, \diff \nu_n^N\\
   &\;
   -\sum_{n=1}^N\int_{\Y} \diff \nu_n^N + 
  \sum_{n=1}^N\int_{\X\times \Y} q(x,y) \, \diff  \mu_n^N(x)\diff\nu^N(y)\\
  = {} & \sum_{n=1}^N \int_{\Y} \log\Big(\frac{1}{\int_{\X} q(x,\cdot) \ \diff\mu_n^N(x)}\frac{\diff\nu_n^N}{\diff\nu^N}\Big) \, \diff\nu_n^N\\
   = {} &-\sum_{n=1}^N \int_{\Y}\log\left(\int_{\X} q(x,\cdot) \ \diff\mu_n^N(x) \right) \diff\nu_n^N + \sum_{n=1}^N \KL(\nu_n^N|\nu^N)\\
  = {}  & N \, J_M^N(q) + \const,
\end{align*}
where we have used \eqref{eq:LossCloudsEmpirical} in the last equation.
\end{proof}

%----------------------------------------------------

\paragraph{Matrix-vector representation.}
For our numerical implementations, we provide, based on Proposition \ref{prop:DiscreteProblem}, a streamlined matrix-vector notation of our problems 
\eqref{eq:DiscreteProblemPre} and \eqref{eq:DiscreteProblemPre1}.
Using the matrix representations 
$$
\zb K_{x} \assign \begin{pmatrix}\zb K_{x}^1\\ \vdots \\ \zb K_{x}^N \end{pmatrix}\in  \R^{MN,K} \quad \text{and} \quad
\zb K_{x}^i\assign
\big( k_{\mu^N \tilde{\mu}}(x_j^i,\tilde x_k) \big)_{j,k=1}^{M,K} ,
$$
and similarly for
$
k_{\nu^N \tilde{\nu}}^\varepsilon
$
the representations 
$\zb K_{y}  \in \R^{MN,L}$ with
$\zb K_{y}^i \in \R^{M,L}$, 
and further 
$\zb{\Xi} \assign \left(\xi_{k,l} \right)_{k,l = 1}^{K,L}$, 
we see that
\[
\zb q \assign \big( \zb q^{n,i}  \big)_{n,i = 1}^{N,N}
= \zb K_x \, \zb \Xi \, \zb K_y^\tT, \quad 
\zb q^{n,i} \assign \big( q(x^n_m,y^{i}_{j} )\big)_{m,j=1}^{M,M}
= \zb K_x^n \, \zb \Xi \, (\zb K_y^i)^\tT.
\]
Then we get the column vector
\[
( A^\varepsilon \xi)_n \, \hat{=} \, \frac{1}{N M^2}  \zb K_{y} \, \zb \Xi^\tT \, (\zb K_{x}^n)^\tT \zb 1_M \in \R^{MN}
\]
as well as
\[
A ^\varepsilon\xi \, \hat{=} \, \frac{1}{N^2 M^2} \zb K_{y} \, \zb \Xi^\tT \, \zb K_x^\tT (I_N \otimes \zb 1_M)\in \R^{MN,M} .
\]
Using the {\rm vec} representation which reorders matrices  columnwise into a vector
and the tensor product $\otimes$ with the property $\text{\rm vec} (A F B^\tT) = (B \otimes A) \text{\rm vec} (F)$, 
this can be rewritten as 
\[
\underbrace{\text{\rm vec} (A^\varepsilon \xi)}_{\zb A^\varepsilon \zb \xi} 
= 
\underbrace{
\frac{1}{N^2 M^2} 
\left( 
\bigl( 
(I_N \otimes  \zb 1_M^\tT)
\zb K_x 
\bigr)
\otimes 
\zb K_y 
\right)
}_{\zb A^\varepsilon \in \R^{N^2 M,KL}} 
\underbrace{\text{\rm vec} 
\, {\zb \Xi}}_{\zb \xi},
\]
which is the matrix representation we will use in the next section.
By  Proposition \ref{prop:EntropicKernels}iv) we know that
$$
\tfrac{1}{MN}\zb K_x^\tT \zb 1_{MN} 
= \zb 1_K , \; \zb K_x \, (\tilde \mu_k)_{k=1}^K = \zb 1_{MN}, 
\quad \text{and} \quad
\tfrac{1}{MN}\zb K_y^\tT \, \zb 1_{MN} =  \zb 1_{L}, \; \zb K_y \, (\tilde \nu_l)_{l=1}^L =  \zb 1_{MN},
$$
so that
\begin{align} \label{matrix}
(\zb A^\varepsilon) ^\tT \zb 1_{N^2M} 
&= 
\frac{1}{N^2M^2} \left(  \zb K_x^\tT  (I_N \otimes  \zb 1_M) \otimes \zb K_y^\tT \right) \, \big(\zb 1_{N} \otimes \zb 1_{MN} \big)\\
&= 
\frac{1}{N^2M^2}
\zb K_x^\tT  (I_N \otimes  \zb 1_M) \zb 1_{N} \otimes  \zb K_y^\tT  \zb 1_{MN}\\
&=
\frac{1}{N^2M^2}
\zb K_x^\tT \zb 1_{MN} \otimes  \zb K_y^\tT  \zb 1_{MN}
= \zb 1_{KL}. \label{propA}
\end{align}
Then, our minimization functional becomes 
\begin{equation} \label{J_discrete}
J_{\text{discr}} (\zb \xi) \assign J^N_M(k^\varepsilon_{\mu^N\tilde{\mu}^N}.\xi.k^\varepsilon_{\nu^N\tilde{\nu}^N}) - \text{const}
= \text{KL} (\zb \zeta, \zb A^\varepsilon \zb \xi),
\end{equation}
where
\begin{equation} \label{zeta}
\zb \zeta \assign \tfrac{1}{N} (\zb \nu ^N_n)_{n=1}^N \in \R^{N^2 M}
, \quad 
\zb \nu ^N_n \assign \big( \delta_{i,n} \zb 1_M\big)_{i=1}^N \in \R^{N M}.
\end{equation}
Finally,  problem \eqref{eq:DiscreteProblemPre}  becomes
\begin{equation} \label{problem1}
 \min_{\xi \in \triangle_{KL} }  \text{KL} (\zb \zeta, \zb A^\varepsilon \zb \xi)
\end{equation}
and its constrained version \eqref{eq:DiscreteProblemPre1} reads as
\begin{equation} \label{problem2}
 \min_{\xi \in \triangle_{KL} }  \text{KL} (\zb \zeta, \zb A^\varepsilon \zb \xi)
 \quad \text{subject to} \quad
 \sum_{l=0}^{L-1} \xi_{k+Kl} = \tilde \mu_k, \quad k=1,\ldots,K.
\end{equation}

\paragraph{Choice of \texorpdfstring{$\tilde {\mathcal X}$}{X} and \texorpdfstring{$\tilde {\mathcal Y}$}{Y}.}
If all sampled points $(\zb x^i,\zb y^i) = ((x_{j}^i,y_{j}^i)_{j=1}^M)_{i=1}^N$ are distinct,
then the total number of points in $\mathcal X^N$ and $\mathcal Y^N$ is $M N$. Choosing then $\tilde{\mu}=\mu^N$ and $\tilde{\nu}=\nu^N$, i.e.~$\tilde{\mathcal X}= \mathcal X^N$, $\tilde{\mathcal Y}= \mathcal Y^N$, the computational complexity of our scheme would grow very quickly.
Instead we can use subsets of $\mathcal X^N$ and $\mathcal Y^N$ as choices for $\tilde{\mathcal X}$ and $\tilde{\mathcal Y}$. This reduces the dimension of the set $\Xi ^N_i$ and the complexity of computing the entropic transport kernels. As long as these sets are asymptotically dense in $\X$ and $\Y$, and $\tilde{\mu}^N$ and $\tilde{\nu}^N$ are chosen consistently, we can still expect our scheme to work well by Proposition \ref{prop:KernelsQNConvergence} and again Proposition \ref{prop:RN}.

Intuitively, for fixed $\varepsilon>0$, the functions $k_{\mu^N \tilde{\mu}}^\varepsilon(\cdot,x)$ for fixed $x$ are approximately Gaussian bumps with a width on the order $\sqrt{\varepsilon}$ and similarly for $k_{\nu^N \tilde{\nu}}^\varepsilon$. 
Therefore, if the Hausdorff distance between $\tilde{\mathcal X}$, 
$\tilde{\mathcal Y}$ and $\X$, $\Y$, respectively, is somewhat less than $\sqrt{\varepsilon}$, then intuitively the richness of the hypothesis class $Q^N$ will not further increase substantially, when increasing the sets $\tilde{\mathcal X}$, $\tilde{\mathcal Y}$. 
This gives a practical guideline on the required size of these sets.

In general, complexity could be reduced further by additional methods. For instance, $\mathcal X^N$ and $\mathcal Y^N$ could be approximated by subsampling and subsequent clustering. Again, we expect that this will not substantially affect the accuracy of the results as long as the Hausdorff distance of the subsampling to the full set is less than $\sqrt{\varepsilon}$. However, for our numerical experiments we found the two above choices sufficient to obtain practically tractable problems.

%--------------------------------------------------
\section{EMML Algorithms}\label{sec:Algorithm}
%--------------------------------------------------
In this section, we discuss algorithms for solving our discrete minimization problems
from the previous section.
We will see that they are just special cases of a more general setting, 
which we adress in this section.
Let $N,I,J,L\in \N$ and $\dot\cup_{l=1}^L J_l =[J]$. 
Furthermore, let 
\begin{align} \label{b}
b &= (b_i)_{i=1}^I \in \R_{> 0}^I
\quad \text{with} \quad
b^\tT \zb 1_I = N \zb 1_J.
\\
 \label{A}
A &= (A_{i,j})_{i,j=1}^{I,J} \in \R^{I \times J}_{\geq 0}
\quad \text{with} \quad
A \zb 1_J > 0, 
\quad \text{and} \quad
A^\tT \zb 1_I = N \zb 1_J.
\end{align}
Then problem  \eqref{problem1} is a special case of
\begin{equation}\label{eq:basic_0}
    \min_{x \in \triangle_J} 
    \KL(b \mid A x).
\end{equation}
Given $y = (y_l)_{l=1}^L \in \triangle_L$, $y >0$,
problem  \eqref{problem2} is a special case of
\begin{equation}\label{eq:basic}
    \min_{x\in \R_{\geq 0}^J} 
    \KL(b \mid A x), 
    \quad 
    \text{subject to} 
    \quad
    \sum_{j\in J_l} x_j = y_l.
\end{equation}
Note that $x \in \triangle_J$ is automatically fulfilled 
if the constraints hold true, since $y \in \triangle_L$.

An efficient  method for solving  unconstrained minimization problems of the form \eqref{eq:basic_0}
is the  \emph{expectation maximization maximum likelihood} (EMML) Algorithm \ref{alg:c_EMML_0},
see, e.g.,\ \cite{byrne2005,VSK1985} and references therein.
%
%--------------------------------------------------------------
\begin{algorithm}[t]
\begin{algorithmic}
    \State \textbf{Input:} 
    \parbox[t]{300pt}{
    $A$, $b$ fulfilling \eqref{b} and \eqref{A}, resp.,\\
    initial $x^{(0)} \in \R^J_{>0}$}
    \For{$r = 0,1,\ldots$ until a convergence criterion is reached} 
    \State
    $x^{(r+1)} \gets x^{(r)} \odot A^\tT \left(\frac{N^{-1} b}{A x^{(r)}}\right)$.
    \EndFor                            
\caption{EMML algorithm}
\label{alg:c_EMML_0}
\end{algorithmic}
\end{algorithm}
%---------------------------------------------------------------
%
The EMML algorithm is known to converge monotonically 
to the solution of \eqref{eq:basic_0}. 
In particular, all iterates $x^{(r)} \in \triangle_J$ and thus also their limit vector
is in $\triangle_J$. However, the standard EMML algorithm is not able to handle equality constraints as those given in \eqref{eq:basic}. 

There are several existing (modified) EMML algorithms 
to handle linear equality constraints \cite{kt1995,J2004,takai2012}.
They have the desirable property that they can be simply built ontop
of the classic EMML algorithm. 
More precisely, in each iteration, after performing the classic EMML update, 
the obtained update is projected orthogonally onto the feasible set.
However, to retain the motonic increase of the likelihood, 
an associated step-size must be chosen carefully.
This results in performing a linesearch in every iteration 
which is usually relaxed to a step-halving procedure.
In this paper, we propose a \emph{constrained} EMML Algorithm
\ref{alg:c_EMML} for solving \eqref{eq:basic} 
which does not require carrying out a linesearch.
The algorithm essentially relies on the same multiplicative update as the EMML,
followed by a renormalization in each partition cell $J_l$ 
such that the constraint in \eqref{eq:basic} is fulfilled.
Due to this renormalization, 
we can omit the multiplication of $N^{-1}$ in the EMML update.
We have the following convergence result.

%--------------------------------------------------------------
\begin{algorithm}[t]
\begin{algorithmic}
    \State \textbf{Input:} 
    \parbox[t]{300pt}{
    $b$, $A$ fulfilling \eqref{b} and \eqref{A}, resp.,\\
    $y = (y_l)_{l=1}^L \in \triangle_L$ with  $y >0$,\\
       initial $x^{(0)} \in \R^{J}_{>0}$}
    \For{$r = 0,1,\ldots$ until a convergence criterion is reached} 
    \State
    $u^{(r)} \gets A^\tT \left(\frac{b}{A x^{(r)}}\right)$.
    \vspace{0.2cm}
    \State 
    $\lambda^{(r)}_l \gets \frac{1}{y_l} \sum_{j \in J_l} x^{(r)}_j u^{(r)}_j,  \quad l \in [L]$.
    \vspace{0.2cm}
    \State 
    $x_j^{(r+1)} \gets \frac{1}{\lambda^{(r)}_l}x^{(r)}_j u^{(r)}_j, \quad j \in J_l$, $l \in [L]$.
    \EndFor                            
\caption{Constrained EMML algorithm}
\label{alg:c_EMML}
\end{algorithmic}
\end{algorithm}
%---------------------------------------------------------------

\begin{theorem}\label{thm:convergence}   
The sequence $(x^{(r)})_r$ generated by Algorithm \ref{alg:c_EMML}
converges to a minimizer of \eqref{eq:basic}.
Furthermore, it holds
\[
\KL(b \mid A x^{(r)}) \geq \KL(b \mid A x^{(r+1)}), \quad r \in \N.
\]
\end{theorem}

Since we have not found a proof in the literature, we add it in Appendix \ref{subsec:cEMML}.
Our proof modifies a geometric argument 
originally provided by Csisz{\'a}r and Tusn{\'a}dy in \cite{csiszar1984information} 
for the unconstrained EMML algorithm. 
We rely on a condensed version of this argument based on \cite{VSK1985}.

\begin{remark}[Sparsity of $b = \bm{\zeta}$]
We remark that the vector $\bm{\zeta} = \tfrac{1}{N} \oplus_{n=1}^N \bm{\nu}_n^N$,
as defined in \eqref{zeta}, has length $NM^2$.
However, by construction, $N(N-1)M$ of its entries are $0$.
The definition of the $\KL$-divergence allows us to 
delete the corresponding rows of $\bm{\zeta}$ and $\bm{A}^\varepsilon$.
Thus, in practice we are required to solve a minimization problem
of the form $\min_{x} \KL(b \vert Ax)$ for 
$A \in \R^{NM \times KL}$ and $b \in \R^{NM}$.
\end{remark}

%-----------------------------------------------------------
\section{Numerical examples} \label{sec:numerics}
%-----------------------------------------------------------
In this section, we present some numerical examples that demonstrate various aspects of our method.
Subsection \ref{sec:torus} illustrates the role of the parameters $M$, $N$, and $\varepsilon$ on a simple stochastic process on the torus.
In Subsection \ref{sec:coloc}, we give a toy example for the analysis of particle colocalization. Finally, Subsection \ref{sec:gyre} analyses particles that move in a double gyre, demonstrating that the method also works on more complex dynamical systems.
%---------------------------
\subsection{Torus} \label{sec:torus}
%---------------------------
\begin{figure}[bt]
    \centering
    \includegraphics[width = 8cm]{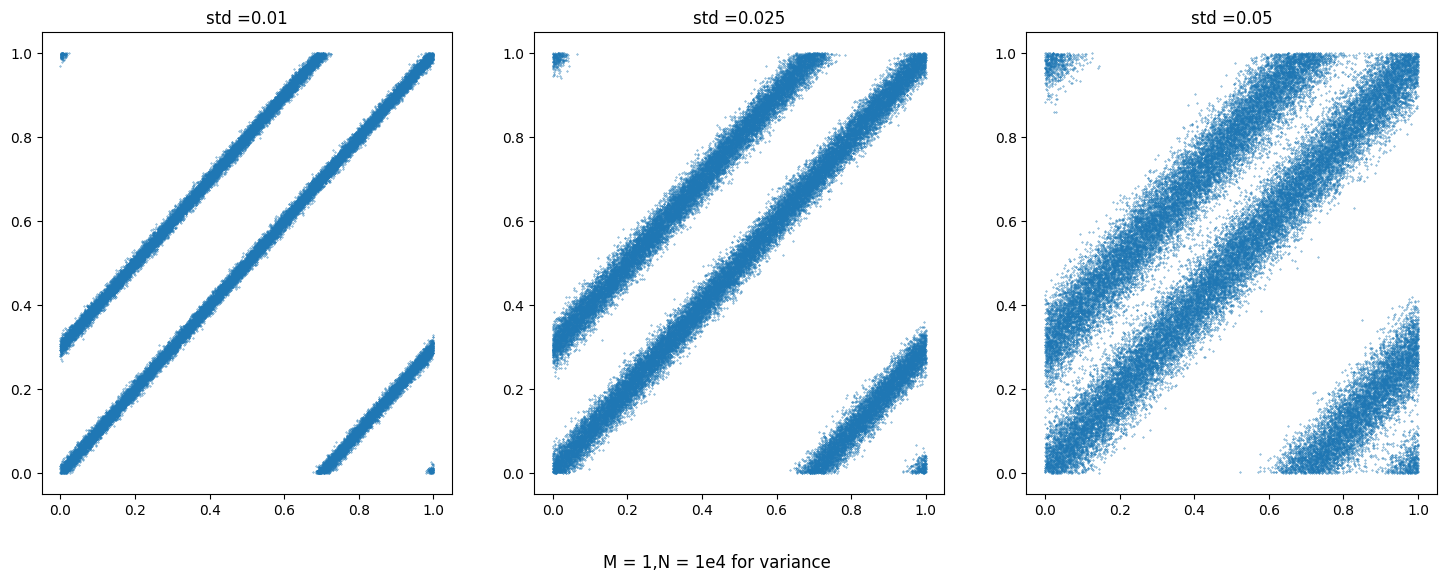}
    \caption{$N=10^4$ samples from $(X,Y)$ according to \eqref{eq:PiTorus} for three different values of $\sigma$.}
    \label{fig:true_sample}
\end{figure}
In the first numerical example, let $\X=\Y=\R/\Z$ be the 1-torus, and let $P_{(X,Y)}=\pi$ be given by
\begin{equation}
    P_X = \mu  \assign  U(\X), \qquad
    P_{Y|X=x}  \assign \frac{1}{2}\tilde{\mathcal{N}}(x, \sigma) + \frac{1}{2}\tilde{\mathcal{N}}(x+0.3, \sigma),
\label{eq:PiTorus}
\end{equation}
where $U(\X)$ is the uniform distribution on $\X$. 
Furthermore, $\tilde{\mathcal{N}}(m,\sigma)$ denotes the wrapped normal distribution with mean $m$ and standard deviation $\sigma$, i.e., $\tilde{\mathcal{N}}(m,\sigma)=f_\# \mathcal{N}(m,\sigma)$, where $\mathcal{N}(m,\sigma) \in \mathcal{P}(\R)$ denotes the normal distribution on $\R$ and $f : x \mapsto x \mod 1$.
That is, conditioned on $X=x$, with probability $1/2$, a sample from $Y$ will be drawn from a (wrapped) Gaussian around $x$, and with probability $1/2$, it will first jump by $0.3$ along the torus.
By symmetry, we deduce that $P_2 \pi =\nu=U(\X)$.
For this $\pi$ we have indeed $\pi=p \cdot (\mu \otimes \nu)$ for a continuous density $p \in \mathcal C(\X \times \X)$.
Samples from $\pi$ for different values of $\sigma$ are shown in Figure \ref{fig:true_sample}.

For given parameters $M$, $N$, $\sigma$, and $\varepsilon$, we then first sample points from $\pi$ according to Assumption \ref{asp:main}.
We set $\mu^N$ and $\nu^N$ as in \eqref{mue}. Furthermore, $\tilde{\mu}, \tilde{\nu}$ are obtained by uniformly sub-sampling 300 points with replacement from each point cloud.
The hypothesis set $Q^N$ is then constructed as in Proposition \ref{prop:KernelsQNConvergence} for $\Xi^N=\Xi^N_{2}$ defined in  Proposition \ref{prop:RN}.
The obtained hypothesis density that minimizes $J_M^N$ over $Q^N$ is denoted by $\hat{q}$. As a measure for the quality of the inference we will use $\|\hat{q}-p\|_{L^2(\mu \otimes \nu)}$ which can be approximated well by numerical integration.
As reference value we will show $\|p-1\|_{L^2(\mu \otimes \nu)}$, i.e.~the discrepancy between the true $p$ and the uniform probability density.

\paragraph{Varying \texorpdfstring{$N$}{N}.}
\begin{figure}[bt]
    \centering
    \includegraphics[width = 14cm]{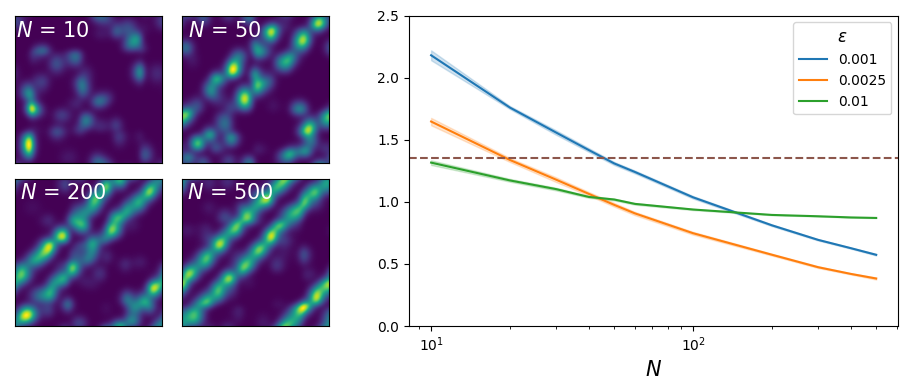}
    \caption{Left: optimal $\hat{q}$ for $\varepsilon=0.0025$ and varying $N$. Right: $\|\hat{q}-p\|_{L^2(\mu \otimes \nu)}$ for $\varepsilon \in \{0.001,0.0025,0.01\}$ and varying $N$. Plot of the mean and standard deviation obtained from 100 simulations.
    We expect that the curve for $\varepsilon=0.001$ will eventually go below the curve for $\varepsilon=0.0025$ as $N$ increases further.
    The dashed line shows $\|p-1\|_{L^2(\mu \otimes \nu)}$.
    In all cases $M=20$, $\sigma=0.05$.}
    \label{fig:sweepN}
\end{figure}
We fix $M=10$, $\sigma=0.05$, and $\varepsilon \in \{0.001,0.0025,0.01\}$ and compute the optimal $\hat{q}$ for different values of $N$. The results are visualized in Figure \ref{fig:sweepN}.
As expected, as $N$ increases and more information becomes available, the reconstruction improves.
Recall that $\sqrt{\varepsilon}$ is roughly the width of the kernels $k^\varepsilon_{\mu^N\tilde{\mu}}$ and $k^\varepsilon_{\nu^N\tilde{\nu}}$.
Consequently there is a bias-variance tradeoff in the choice of $\varepsilon$. For larger $\varepsilon$, as $N$ increases the error $\|\hat{q}-p\|_{L^2(\mu \otimes \nu)}$ initially decreases faster (less variance), but saturates earlier and at a higher level (more bias), than for smaller $\varepsilon$.

\paragraph{Varying \texorpdfstring{$\varepsilon$}{E}.}
\begin{figure}[bt]
    \centering
    \includegraphics[width = 14cm]{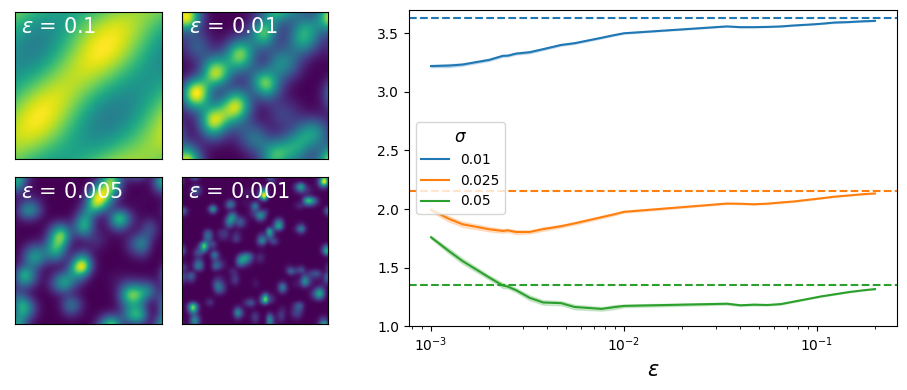}
    \caption{Left: optimal $\hat{q}$ for $\sigma=0.05$ and varying $\varepsilon$. Right: $\|\hat{q}-p\|_{L^2(\mu \otimes \nu)}$ for $\sigma \in \{0.01,0.025,0.05\}$ and varying $\varepsilon$. Shown are mean and standard deviation obtained from 100 simulations. The dashed line indicates $\|p-1\|_{L^2(\mu \otimes \nu)}$ for reference.
    In all cases $M=N=20$.}
    \label{fig:sweepEps}
\end{figure}
Next, we vary $\varepsilon$ for fixed $M=N=20$ and $\sigma \in \{0.01,0.025,0.05\}$ to emphasize the bias-variance tradeoff.
The results are shown in Figure \ref{fig:sweepEps}. For small $\varepsilon$, the reconstructed $\hat{q}$ has many concentrated spikes, also in places where $p$ is close to zero, indicating that the small width prevents efficient extraction of information from the sample data. As $\varepsilon$ increases, the double band structure of $p$ becomes visible, albeit still subject to some local noise. Finally, for even larger $\varepsilon$, the width of $k^\varepsilon_{\mu^N\tilde{\mu}}$ and $k^\varepsilon_{\nu^N\tilde{\nu}}$ becomes too high to resolve the two individual bands in $p$ and merely a single smooth band is reconstructed.
This trend is reflected in the plot of $\|\hat{q}-p\|_{L^2(\mu \otimes \nu)}$ over $\varepsilon$: first, there is a steep decrease and eventually again an increase. As $\varepsilon \to \infty$ the error approaches the value $\|1-p\|_{L^2(\mu \otimes \nu)}$, since $k^\varepsilon_{\mu^N\tilde{\mu}}$ and $k^\varepsilon_{\nu^N\tilde{\nu}}$ converge to $1$ uniformly, and hence so does $\hat{q}$.
Second, for small $\sigma$, the mass of $p$ is concentrated close to the two one-dimensional lines $Y=X$ and $Y=X+0.3$, whereas for larger $\sigma$ it is spread out further over the full two-dimensional torus. Hence, for smaller $\sigma$, fewer datapoints are necessary for reconstruction and it can be approximated less well with large $\varepsilon$. Therefore, as $\sigma$ increases, so does the $\varepsilon$ for which the reconstruction is optimal, and the asymptotic error $\|1-p\|_{L^2(\mu \otimes \nu)}$ decreases.

\paragraph{Varying \texorpdfstring{$M$}{M}.}
\begin{figure}[bt]
    \centering
    \includegraphics[width = 14cm]{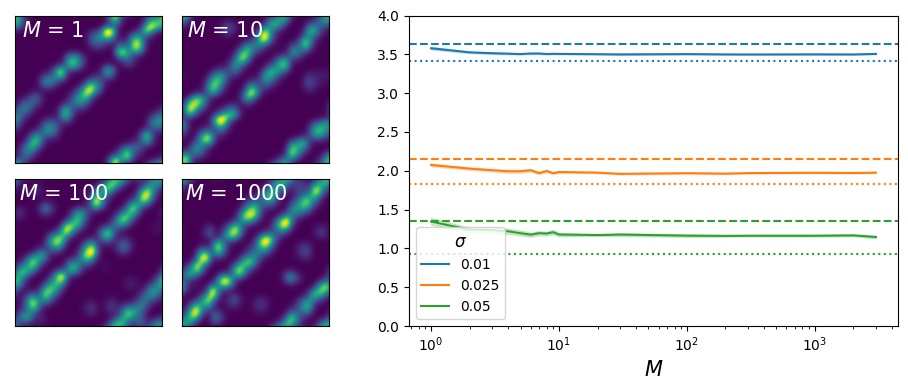}
    \caption{Left: optimal $\hat{q}$ for $N=70$, $\sigma=0.025$, $\varepsilon=0.0025$ and varying $M$. Right: $\|\hat{q}-p\|_{L^2(\mu \otimes \nu)}$ for $N=20$, $\sigma \in \{0.01,0.025,0.05\}$, $\varepsilon =0.01$ and varying $M$. Plot of the mean and standard deviation obtained from 100 simulations. The dashed line indicates $\|p-1\|_{L^2(\mu \otimes \nu)}$. The dotted line shows the error $\|p-\tilde{q}\|_{L^2(\mu \otimes \nu)}$, where $\tilde{q}$ is the minimizer of $J_{M=1}$ over $Q$, i.e.~the best possible approximation of $p$ by $Q$ in the sense of $\KL$ and without the added difficulty of unobserved pairings.}
    \label{fig:sweepM}
\end{figure}
Finally, we consider varying $M$. The results are illustrated in Figure \ref{fig:sweepM}.
We observe that the reconstruction quality first improves with increasing $M$ and eventually decreases slightly.
This means that for moderate $M > 1$, additional information can be extracted from the additional observed points relative to the case $M=1$ despite the fact that their pairing is not known.
For intuition, consider the case of very small $\sigma$. Then for small $M$ it will be relatively easy to guess the correct association of point pairs and thus increasing $M$ will be somewhat similar to increasing $N$.

As expected, the lack of this unobserved association becomes more severe as $M$ increases further, and thus the reconstruction quality does not improve indefinitely, but indeed decreases again slightly.
Intuitively, as $M$ increases, the observed point pairs $(x_j^i,y_j^i)_{j=1}^M$ of one sample of $\zb Z$ in \eqref{Z} become increasingly independent and approximately i.i.d.~samples from the independent distribution $\mu \otimes \nu$. One might therefore expect that inference breaks down for very large $M$. However, this does not seem to be the case and our method even works in this regime.

\paragraph{Parametric inference for varying \texorpdfstring{$M$}{M}.}
\begin{figure}[bt]
    \centering
    \includegraphics[width = 13cm]{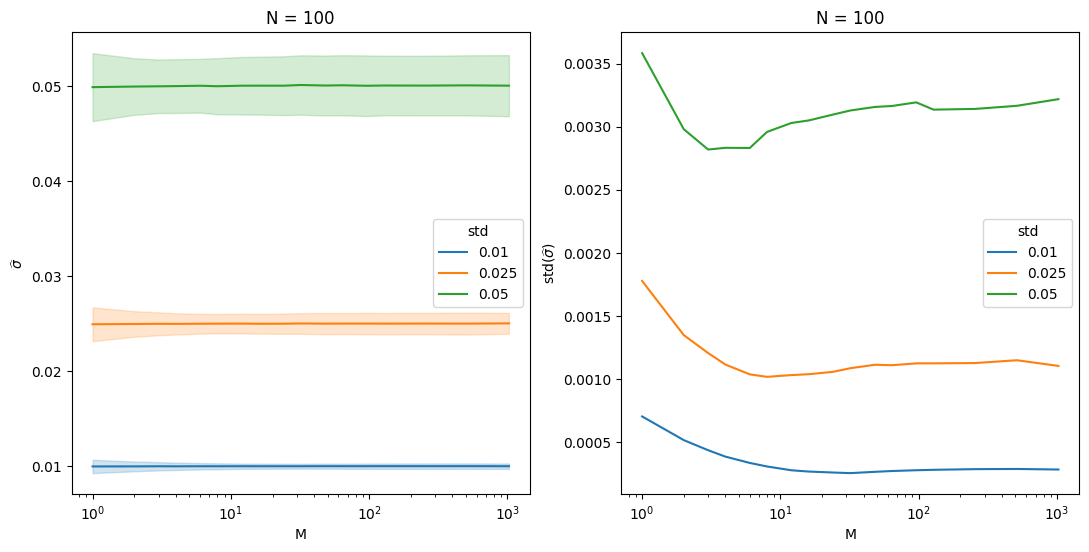}
    \caption{Left: Mean and standard deviation of the optimal $\hat{\theta}$ for three different $\sigma$ and $N=1000$ as estimated from 2000 simulations. Right: Standard deviation shown separately for emphasis.} 
    \label{fig:FI}
\end{figure}
We further analyze the effect of varying $M$ in a simpler parametric setting.
For $\theta>0$, denote by $\pi_\theta$ the measure specified by \eqref{eq:PiTorus} with $\sigma=\theta$. Then set $q_\theta \assign \tfrac{\diff \pi_\theta}{\diff \mu \otimes \nu}$, and $Q \assign \{ q_\theta | \theta >0 \}$. In this case, the optimization over $Q$ can be formulated as optimization over the one-dimensional parameter $\theta$. We denote the optimal value by $\hat{\theta}$. This could include $\hat{\theta}=\infty$ which would correspond to the uniform density. 
By Proposition \ref{prop:PopMin}, in the limit $N \to \infty$ the unique minimizer is then $\hat{\theta}=\sigma$.
Figure \ref{fig:FI} shows mean and standard deviation of the estimator $\hat{\theta}$ for $N=100$, $\sigma \in \{0.01,0.025,0.05\}$ and varying $M$ as estimated over 2000 simulations.
The average of the estimated $\hat{\theta}$ is correct for all $\sigma$, indicating that $\hat{\theta}$ is unbiased. For the standard deviation, we observe a behaviour consistent with the previous paragraph. As $M$ increases, the standard deviation first decreases and then increases again slightly, but remains bounded and seems to approach a stable limit as $M \to \infty$.
Drawing intuition from the Bernstein--von Mises theorem, this seems to suggest that the curvature of the functional $J_M^N$ at $p$ does not degenerate to zero as $M \to \infty$ and a non-zero amount of information can be extracted from each sample of $\zb Z$ even for large $M$.
We will further study this phenomenon in future work.

%------------------------------------
\subsection{Particle colocalization}
\label{sec:coloc}
%------------------------------------
\begin{figure}[bt]
    \centering
    \includegraphics[width = 13cm]{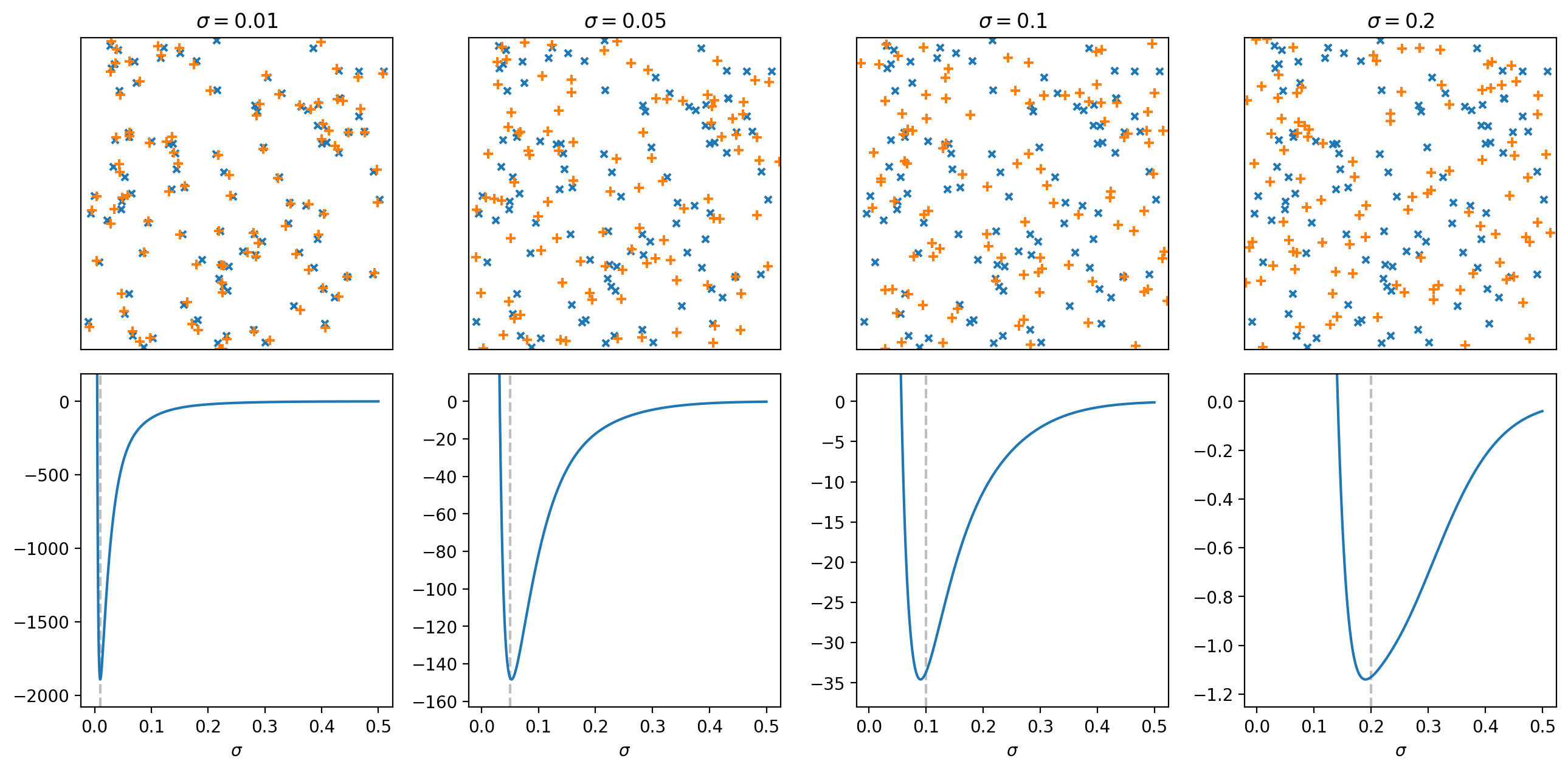}
    \caption{Top row: Single samples $(x_j,y_j)_{j=1}^M$ of $\zb Z$, \eqref{Z}, for $M=100$ and $\sigma \in \{0.01,0.05,0.1,0.2\}$, shown as point clouds on $\X$ (represented by $[0,1]^2$). Bottom row: Corresponding objective $J_M^N(q_\theta)$ for $\theta > 0$ for a single simulation of $N=10$ samples.}
    \label{fig:FI_1}
\end{figure}
Next, we use a variant of the above torus example as a toy model to illustrate the potential of our method for particle colocalization analysis.
Similar to \eqref{eq:PiTorus}, let $\X=\Y=(\R/\Z)^2$ be the 2-torus and let $P_{(X,Y)}=\pi$ be given by
\begin{align*}
    P_X = \mu & \assign  U(\X), &
    P_{Y|X=x} & \assign \tilde{\mathcal{N}}(x, \sigma),
\end{align*}
where  $\tilde{\mathcal{N}}(x, \sigma)$ denotes a wrapped isotropic normal distribution centered at $x$ with standard deviation $\sigma$ in all directions. Similar as in the last paragraph, we consider inference of $\sigma$ via a parametric family $Q\assign \{q_\theta | \theta>0\}$.
Figure \ref{fig:FI_1} shows individual samples from $\zb Z$ for $M=100$ and various $\sigma$, as well as corresponding curves $\theta \mapsto J_M^N(q_\theta)$ for $N=10$.
As can be seen, for small $\sigma$ the association of the point pairs can be guessed correctly with high probability due to their high spatial proximity, and consequently $\sigma$ can easily be inferred correctly. For larger $\sigma$, it seems visually impossible to reliably recover the point pairings. Nevertheless, the curves $J_M^N(q_\theta)$ have pronounced minima at the correct value of $\theta$ in all cases.

%------------------------------------
\subsection{Double gyre}
\label{sec:gyre}
%------------------------------------
We now consider the following well-studied (see for instance \cite{banisch2017understanding,
froyland2014almost,
OTCoherentSet2021})
deterministic non-autonomous system: 
\begin{align}\label{eq:double_gyre}
    \frac{\diff x_1}{\diff t} 
    &= - \pi A \sin\big(\pi f(t,x_1)\big) \cos(\pi x_2) \nonumber\\
    \frac{\diff x_2}{\diff t} 
    &= \pi A \sin\big(\pi f(t,x_1) \big) \sin(\pi x_2) 
    \frac{\partial f}{\partial x_1} (t,x_1),
\end{align}
where 
$f(t,x) \assign \alpha \sin(\omega t) x^2 + \left(1 - 2\alpha \sin(\omega t) \right) x$, 
$A = \alpha \assign 0.25$ and $\omega \assign 2 \pi$.
The system describes two adjacent counter-rotating gyres 
in $\X = \Y \coloneqq [0,2] \times [0,1]$, 
see Figure \ref{fig:double_gyre}. 
As can be seen in the figure, 
the vertical boundary between the gyres
oscillates periodically and takes exactly $\Delta t = 1$
to move from the leftmost positon at approximately $1 - \alpha$
to the rightmost position at approximately $1 + \alpha$ ($x$-direction).
\begin{figure}
\centering
\begin{subfigure}{0.4\textwidth}
    \centering
    \includegraphics[width=\textwidth]{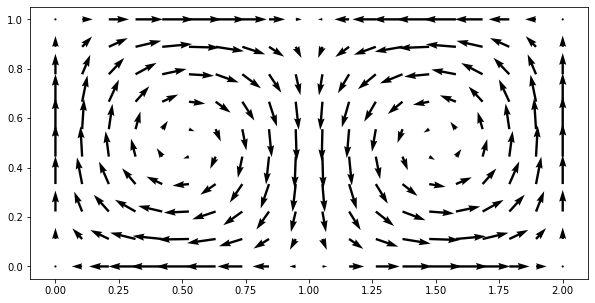}
    \caption{$t = 0.5 \, k$, $k \in \N_0$}
\end{subfigure}
\\
\begin{subfigure}{0.4\textwidth}
    \centering
    \includegraphics[width=\textwidth]{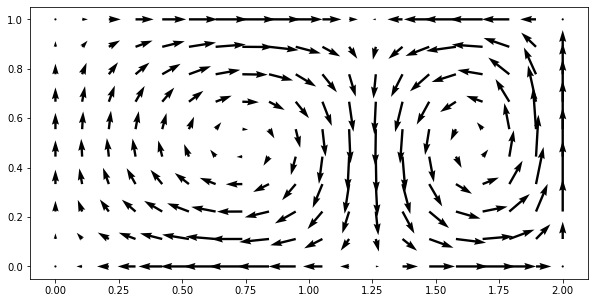}
    \caption{$t = 0.25$}
\end{subfigure}
\begin{subfigure}{0.4\textwidth}
    \centering
    \includegraphics[width=\textwidth]{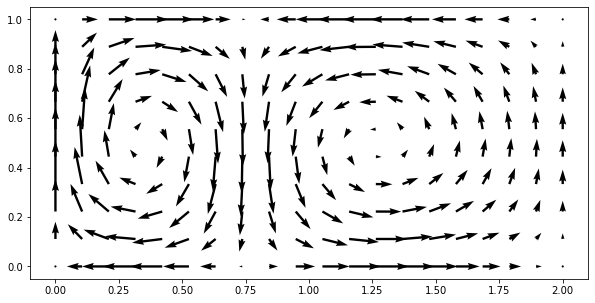}
    \caption{$t = 0.75$}
\end{subfigure}
\caption{The velocity field 
$(\frac{\diff x}{\diff t},\frac{\diff y}{\diff t})$
of \eqref{eq:double_gyre} at various times $t$.}
\label{fig:double_gyre}
\end{figure}
Figuratively speaking, 
particles on the left-hand side of the boundary
rotate around the left gyre and likewise for the right gyre.
Within a gyre, the particles distance 
to the gyre center remains unchanged for long periods,
and particles rarely transition towards the gyre boundary or
towards the gyre center.
In addition, particles close to vertical boundary between the gyres
rarely transition to a rotation around the other gyre.

For a fixed time-step $\Delta t$, 
the system can be described by a 
map $T:\X \to \X$, 
where $y = T(x) \in \X$ 
is the position of a particle $x =(x_1,x_2) \in \X$ 
after the time-step $\Delta t$,
when it is evolved according to \eqref{eq:double_gyre}.
More precisely, 
\[
T(x) = (\hat{x}_1(\Delta t),\hat{x}_2(\Delta t))
\]
where
$(\hat{x}_1(t),\hat{x}_2(t))$
is a solution to \eqref{eq:double_gyre} with initial condition
$(\hat{x}_1(0),\hat{x}_2(0)) = (x_1,x_2)$.
Let $\ell$ be the uniform distribution on $\X$
and $X$ be a random variable on $\X$ with law $\ell$.
We are interested in the joint measure 
$\pi = (\text{id},T)_\# \ell$ of
$(X,Y)$ with $Y = T(X)$.
The dynamical system preserves the 
Lebesgue measure, i.e.\ it holds
$\pi \in \Pi(\ell,\ell)$.
In this case, $\pi$ has no density with respect to its marginals (which are both $\ell$) but minimizing $J_M^N$ and $J_M$ are still well-posed problems.
The disintegration of $\pi$ with respect to its first marginal at $x$ is given by $\delta_{T(x)}$
and the associated transfer operator is
\[
\mathcal{T}:L^1(\ell) \to L^1(\ell), \quad
(\mathcal{T}u) \coloneqq u \circ T^{-1}
\]
where we use that $T$ is invertible since it is defined as flow of a sufficiently regular vector field.

\paragraph{Sampling and computation of \texorpdfstring{$q$}{q}.}
We are seeking to estimate the transfer operator 
$\mathcal{T}$ for $\Delta t = 3$.
To do this, we generate $N=300$ initial measures 
$\mu_i^N  \assign \frac{1}{M} \sum_{j=1}^{M} \delta_{x^i_j}$, 
$i=1 \ldots, N$ 
supported on $M=50$ uniformly sampled points
$(x^i_j)_{j=1}^M$ on $\X$ each.
The target measures are constructed 
by setting $y^i_j \coloneqq T(x^i_j)$ and then
$\nu_i^N  \assign \frac{1}{M} \sum_{j=1}^{M} \delta_{y^i_j}$.
We set $\mathcal X^N \assign \{x_j^i: i \in [N], j \in [M]\}$ and
$\mathcal Y^N \assign \{y_j^i: i \in [N], j \in [M]\}$.
Furthermore, we generate $\tilde{\mathcal X}$
by sampling $K=1000$ points $\tilde x_k \in \mathcal X^N$, and similarly $\tilde{\mathcal Y}$ by sampling 
$L=1000$ points $\tilde y_l \in \mathcal Y^N$
in a furthest-point manner.
On $\tilde{\mathcal{X}}$, we consider the probability measure
$\tilde{\mu} \assign \sum_{k=1}^K \tilde \mu_k \delta_{\tilde x_k}$ with
\[
\tilde{\mu}_k \coloneqq 
\frac{
\lvert \{x \in \mathcal{X}^N : 
\|x - \tilde{x}_k \| < \|x - \tilde{x}_{k'}\|, k' \neq k\}\rvert
}
{NM}.
\]
In other words, the weights of the respective points of $\tilde{\mu}$ 
are proportional to the number of closest neighbours in $\mathcal{X}^N$.
We construct $\tilde \nu \in \mathcal{P}(\tilde{\mathcal{Y}})$ analogously.
Now we compute the kernels 
$k_{\mu^N\tilde{\mu}}^\varepsilon$ 
via entropic OT between $\mu^N$ and $\tilde \mu$ 
for $\varepsilon = \sqrt{5} \cdot 0.01 \approx 0.022$, 
and similarly  $k_{\nu^N\tilde{\nu}}^\varepsilon$.
For
$
\Xi_2 \assign
\{\xi \in \mathcal P(\tilde{\mathcal{X}} \times \tilde{\mathcal{Y}}) 
: P_1 \xi = \tilde{\mu}\}$,
we solve
\[
\argmin_{\xi \in \Xi_2} 
J^N_M(k_{\mu^N\tilde{\mu}}^\varepsilon . \xi . k_{\nu^N\tilde{\nu}}^\varepsilon)
\]
by Algorithm \ref{alg:c_EMML}.
The obtained results $\hat{\xi}$ 
as well as 
$\hat{q}
\coloneqq 
k_{\mu^N\tilde{\mu}}^\varepsilon . \hat{\xi} . k_{\nu^N\tilde{\nu}}^\varepsilon$ are illustrated in Figures~\ref{fig:double_gyre_result_1}~and~\ref{fig:double_gyre_result_2}, respectively.
Notably, the inferred transfer operator reads
\[
\hat{\mathcal{T}} u \coloneqq \int_{\X} \hat q(x,\cdot) u(x) \, \diff \ell(x).
\]
We remark that by construction, 
the measures $\mu^N,\tilde{\mu}$ are discrete approximations 
of the uniform distribution $\ell$.
In addition, since $T$ preserves $\ell$ (i.e.\ $T_\# \ell = \ell$), 
the same holds true for $\nu^N,\tilde{\nu}^N$.
Thus, the assumptions of 
Proposition~\ref{prop:KernelsQNConvergence}, 
Corollary~\ref{cor:Gamma} 
and Proposition~\ref{prop:RN} are fulfilled
and $\hat{q}$ is a discrete approximation of $\argmin J_M + \iota_{Q}$,
where
\[
Q \coloneqq \{k^\varepsilon_{\ell,\ell} . \xi . k^\varepsilon_{\ell,\ell} : {P_1}_\# \xi = \ell\}.
\]
Thus, the inferred density $\hat{q}$ 
appears to  be a blurred approximation of the ground truth $\pi$. 
This is also observed in Figure\ref{fig:double_gyre_result_2}.

\begin{figure}[t]
    \centering
    \includegraphics[height = 130pt]
    {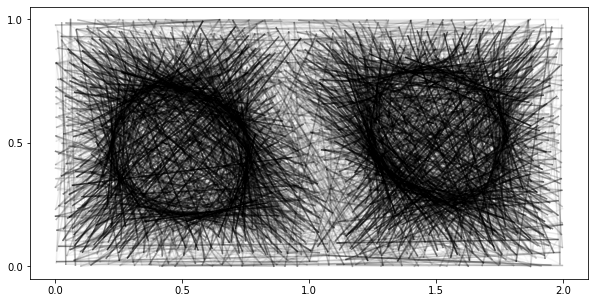}
    \caption{
    The opacity of the lines is proportional 
    to the transport $\hat \xi_{kl}$ 
    between the endpoints $\tilde x_k$ and $\tilde y_l$.   
    The two adjacent gyres can be clearly identified.
    }
    \label{fig:double_gyre_result_1}
\end{figure}

\begin{figure}
    \centering
    \includegraphics[width=\linewidth]{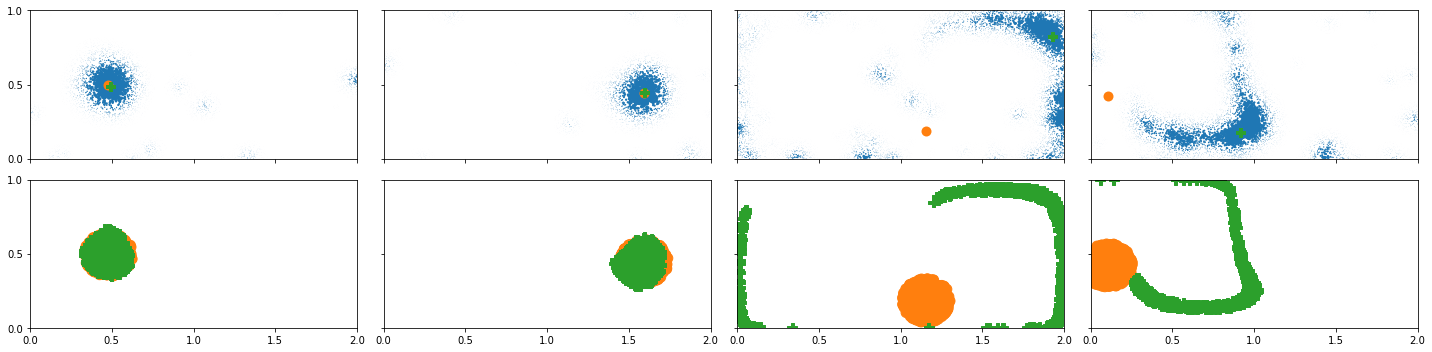}
    \caption{Top row:
    Four different points $x^i_j \in \mathcal{X}^N$ (orange),
    their propagation according to $\hat{\mathcal{T}}$,
    i.e. $q(x^i_j,\cdot)$ (blue)
    and the ground truth propagation $y^i_j = T(x^i_j)$ (green).
    It is well reflected that mass at points 
    close to the gyre centers remains there 
    with high probability.
    Mass at the boundary of each gyre mostly persists there, 
    but can also transition to the other gyre.
    Bottom row:
    The subset of points of $\mathcal{X}^N$ 
    which are at most $\sqrt{\varepsilon}$
    away from the  corresponding points $x^i_j$ picked above (orange)
    and their ground truth propagations (green).
    The similarity of the top blue point cloud
    and the corresponding bottom green point cloud shows that, up to the inevitable blurr of $\sqrt{\varepsilon}$,
    the inferred transfer operator $\hat{\mathcal{T}}$ provides
    a good approximation of the ground truth $\mathcal{T}$.}
    \label{fig:double_gyre_result_2}
\end{figure}

\paragraph{Clustering.} 
To determine coherent structures, 
we perform a spectral clustering procedure on $q$.
The idea is to find partitions 
$\mathcal X^N =\mathcal   X_1^N \dot \cup \mathcal  X_2^N$, $\mathcal Y^N = \mathcal Y_1^N \dot \cup \mathcal Y_2^N$,
such that
\[
\mu^N(\mathcal X^{N}_\kappa) = \nu^N(\mathcal Y^{N}_\kappa) \quad \text{and} 
\quad 
\hat{\mathcal{T}} 1_{\mathcal X^{N}_\kappa} = 1_{\mathcal Y^{N}_\kappa} , \quad \kappa = 1,2.
\]
We refer to \cite{FROYLAND20131,OTCoherentSet2021} for the following facts.
The partition problem is equivalent to solving the minimization problem
\begin{equation}\label{eq:hard}
\max_{
\mathcal X_1 \dot \cup \mathcal X_2 = \mathcal X,
\mathcal Y_1 \dot\cup\mathcal Y_2 = \mathcal Y}  
\biggl\{ 
\frac{\langle \hat{\mathcal{T}} \zb 1_{\mathcal X_1}, 1_{\mathcal Y_1} \rangle_\nu}{\mu(\mathcal X_1)} 
+ \frac{\langle \hat{\mathcal{T}} 1_{\mathcal X_2}, 1_{\mathcal Y_2} \rangle_\nu}{\mu (\mathcal X_2)} 
\biggr\}.
\end{equation}
Numerically this problem is challenging. 
Hence it is usually relaxed to
\begin{equation}\label{eq:fuzzy}
\max_{(\varphi,\psi) \in L^2_\mu(\mathcal X) \times L^2_\nu(\mathcal Y)} 
\biggl\{
\frac
{\langle \hat{\mathcal{T}} \varphi,\psi\rangle_\nu}
{\|\varphi\|_{\mu} \|\psi\|_{\nu}}
: \langle \varphi, 1_{\mathcal X} \rangle_\mu 
= \langle \psi, 1_{\mathcal Y}\rangle_\nu 
= 0
\biggr\}.
\end{equation}
A maximizing pair $(\hat{\varphi},\hat{\psi})$ in \eqref{eq:fuzzy} 
is given by the right and left singular functions of $\hat{\mathcal{T}}$ 
associated to the second largest singular value of $\hat{\mathcal{T}}$.
An approximate solution of \eqref{eq:hard} is then provided by
thresholding $\hat{\varphi},\hat{\psi}$ at $0$.

The clustering procedure is used in Figure \ref{cluster} to segment the different gyres and to distinguish between gyre centers and their boundaries. 
The left and right singular vectors corresponding to the second largest singular value,
shown in Figure \ref{subfig:double_gyre_1}, 
indeed reveals the two coherent sets given by the two gyres, respectively.
Furthermore, 
the left and right singular vectors corresponding to the third largest singular value, 
shown in Figure \ref{subfig:double_gyre_2},
allow us to further distinguish between the joint boundary and the joint centers.

\begin{figure}
    \begin{subfigure}{1\textwidth}
    \centering
    \includegraphics[height=3cm]{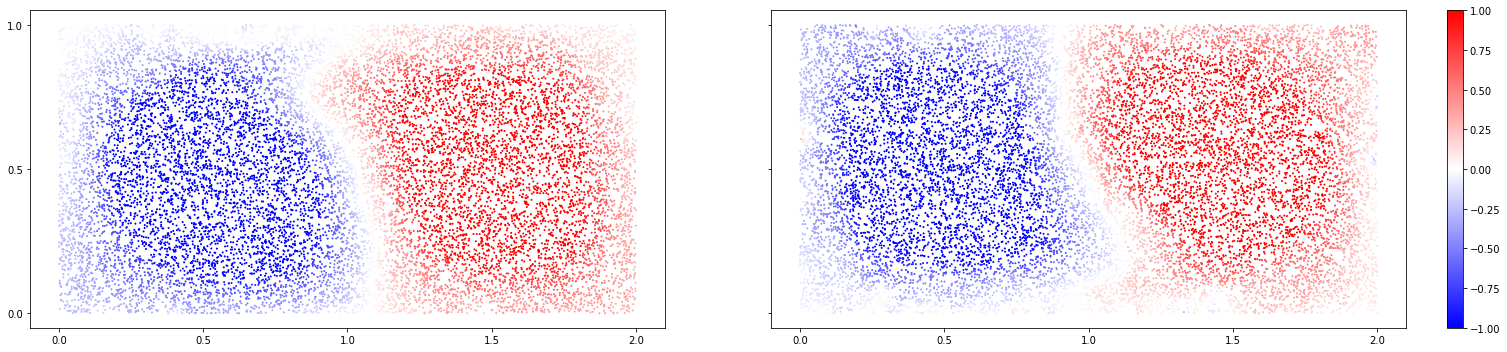}
    \subcaption{Left and right singular vectors for the second largest singular value.
    The partition distinguishes between the two gyres.}
    \label{subfig:double_gyre_1}
    \end{subfigure}
    \\
    \begin{subfigure}{1\textwidth}
    \centering
    \includegraphics[height=3cm]{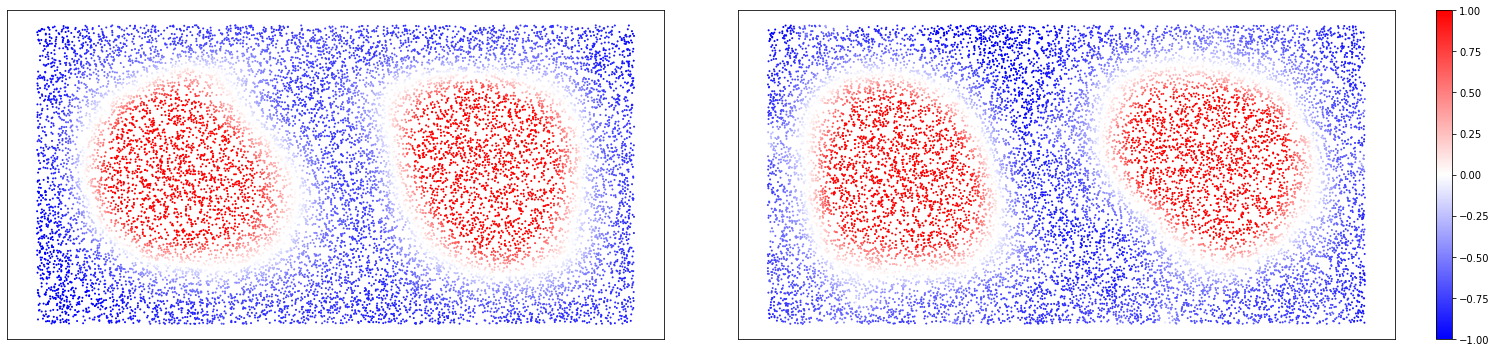}
    \subcaption{Left and right singular vectors for the third largest singular value.
    The partition distinguishes between gyre centers and gyre boundaries.}
    \label{subfig:double_gyre_2}
    \end{subfigure}
    \caption{Spectral clustering on the obtained solution $\hat{q}$.} \label{cluster}
\end{figure}

%-------------------------------------------------
\section{Conclusion and outlook} \label{sec:Conclusion}
%-------------------------------------------------
We have proposed an inference model for learning from various observations of batches of unpaired samples.
In particular, we proposed as hypothesis density space that employs kernels from entropic optimal transport.
After discretization, we solved the arising minimization problem by an extenced EMML algorithm
and illustrated the potential of our approach with numerical experiments.

An open question for future work is a quantitative analysis how much information can be extracted 
from each batch sample $(x^{i}_j,y^i_,j)_{j=1}^M$ for a single $i$. 
As we have seen in the numerical examples, for small $M$, 
additional unpaired samples seem to increase the information per batch 
and to improve the inference performance for fixed $N$ and $\varepsilon$. 
For larger $M$, the probabilistic inference of the particle pairings becomes too difficult 
and the information content per batch drops. 
In the formal limit $M \to \infty$, 
the empirical distributions of the observations $(x^{i}_j)_{j=1}^M$ and $(y^{i}_j)_{j=1}^M$ will converge to $\mu$ and $\nu$. 
In this case, no inference appears to be possible. 
It seems therefore surprising that our results 
suggest that inference can be viable even for very large $M$, given enough batches $N$. 
We expect that this phenomenon can be studied analytically in the spirit of the Bernstein--von Mises theorem.

Another question is related to the parameter $\varepsilon$ in the entropic transport kernels.
Intuitively, it controls complexity and flexibility of hypothesis class $Q$ 
and its choice will depend on a bias-variance trade-off related to the amount of available data. 
In this paper, we kept $\varepsilon>0$ fixed and only qualitatively considered the limit as $N \to \infty$. 
In future work, we will derive quantitative rates and examine the joint limit $(\varepsilon,N) \to (0,\infty)$.

\section*{Acknowledgement}
F.B. and G.S. acknowledge the funding by the German Research Foundation (DFG) within the RTG 2433 DAEDALUS, and G.S. by the
DFG project STE 571/19-1 
within the Austrian SFB ,,Tomography across scales''.
H.B., C.S, and B.S. acknowledge the funding by the German Research Foundation (DFG) within the CRC 1456, Mathematics of Experiment, projects A03 and C06, and the Emmy Noether-Programme.

We thank J.~Hertrich and P.~Koltai for discussions on modeling and presentation.

%---------------------------------------------------------
\appendix
%--------------------------------------------------------
\section{Proof of Theorem \ref{thm:convergence}}\label{subsec:cEMML}
%--------------------------------------------------------
The proof of Theorem \ref{thm:convergence} requires various technical lemmata and propositions.
In the following, let
\[
R \coloneqq \{x \in \R_{\ge 0}^J:  \sum_{j \in J_l} x_j = y_l
\text{ for all } l \in [L]\},
\]
which is a closed, convex and nonempty set.
Let 
$$F(x) \coloneqq - \langle b, \log(Ax) \rangle, \quad x \in R.$$
By assumption, 
we have $\sum_{i=1}^N (A x)_i = N \sum_{i=1}^N x_i = N$ for all $x \in R$. 
Hence, the optimzation problem \eqref{eq:basic} is equivalent to
\begin{equation}\label{eq:basic_a}
\min_{x \in \R^J} F(x) 
\quad \text{subject to} \quad \sum_{j \in J_l} x_j = y_l
\text{ for all } l \in [L],
\end{equation}
where we agree that $\log 0 = +\infty$.
Without loss of generality, we assume that $b_i > 0$ 
for all $i \in [I]$.
As before, let $\iota_R$ denote the indicator function of $R$ which is zero on $R$ 
and $+\infty$ outside of $R$.
Since the objective  
$F + \iota_R$ 
in \eqref{eq:basic_a}
is proper, lower semi-continuous and coercive, the minimization problem has a solution.

We want to solve the minimization problem \eqref{eq:basic_a} by a majorization-minimization algorithm.
To this end, we define for a fixed $\eta \in \R^J_{> 0}$, the function 
$G(\cdot,\eta): \mathbb R^I \to \mathbb R \cup \{+\infty\}$ by
\begin{equation}\label{eq:surrogate}
G(x,\eta) \coloneqq 
\sum_{i=1}^I b_i \sum_{j=1}^J g_{i,j}(\eta) \log \left(\frac{ g_{i,j}(\eta) }{ A_{i,j} x_j } \right), 
\quad  g_{i,j}(\eta) \coloneqq \frac{A_{i,j} \eta_j}{(A \eta)_i},
\end{equation}
where we set $G(x,\eta) \coloneqq + \infty$ if $x_j \le 0$ for some $j \in [J]$.
The following proposition summarizes properties of $G$.

\begin{proposition}\label{prop:surrogate}
\begin{itemize} 
\item[i)]
For any $\eta \in \R^J_{>0}$, the function $G(\cdot,\eta)$ in \eqref{eq:surrogate}
is strictly convex. The function
$G(\cdot,\eta) + \iota_R$ has a unique minimizer which is an element of $R^J_{> 0}$ and is given by
\begin{equation} \label{eq:min_sur}
x_j = \frac{1}{\lambda_l} \eta_j u_j, \; j \in J_l
\quad 
 \lambda_l \coloneqq \frac{1}{y_l} \sum_{j \in J_l} \eta_j u_j, 
\quad 
u \coloneqq  A^\tT \left(\frac{b}{A \eta}\right)                            
\end{equation}
where the quotient is understood componentwise.
\item[ii)]
The function
$G: \R_{>0}^J \times \R_{>0}^J \to \mathbb R \cup \{+\infty\}$ is a surrogate function
of $F$ in \eqref{eq:basic_a}, i.e.\
$G(x,x) = F(x)$ and $G(x,\eta) \ge F(x)$ for any $\eta \in \R^J_{>0}$
\end{itemize}
\end{proposition}

\begin{proof}
i) Let $\eta \in \R^J_{>0}$ be arbitrary but fixed.
Let $x,\tilde{x} \in R$ with $x \neq \tilde{x}$. 
Then there exists $j \in [J]$ so that $x_{j} \neq \tilde{x}_{j}$
and $i \in [I]$ with $A_{i,j} > 0$.
By strict concavity of the logarithm, we obtain for $t \in (0,1)$ that
\[
\log\left(\frac{A_{i,j}}{g_{i,j}(\eta)} \left( (1-t)x_j + t\tilde{x_j}\right) \right)
>
(1-t)\log\left(\frac{A_{i,j}}{g_{i,j}(\eta)} x_j \right) 
+ 
t\log\left(\frac{A_{i,j}}{g_{i,j}(\eta)}\tilde{x_j}\right),
\]
which implies by definition that $G(\cdot,\eta)$ is strictly convex.
Since $G + \iota_R$ is also lower semi-continuous and coercive, it has a unique minimizer $x$ which must fulfill componentwise
$x >0$ by the definition of $G$.                                                                                         
By the KKT condition with the Lagrangian 
$\mathcal L(x,\lambda) \coloneqq G(x,\eta) + \sum_{l=1}^L \lambda_l (y_l - \sum_{j\in J_l} x_j)$
this minimizer is determined by $\nabla_x \mathcal L(x,\lambda) = 0$ and $\nabla_\lambda \mathcal L(x,\lambda) = 0$
which means
\begin{align} \label{eq:lagr}
   - \sum_{i=1}^I b_i g_{i,j}(\eta)  \frac{1}{x_j} + \lambda_l &= 0 \quad \text{for all } j \in J_l, \nonumber\\
    \sum_{j \in J_l} x_j &= y_l \quad \text{for all }  l \in [L].
\end{align}
Multiplying with $x_j$ and adding up the first equations over $j \in J_l$ results in
$$
\sum_{j \in J_l} \sum_{i=1}^I b_i g_{i,j}(\eta) = \sum_{j \in J_l} \lambda_l x_j = \lambda_l y_l,
$$
which yields
\begin{align*}
\lambda_l 
&= \frac{1}{y_l}\sum_{j \in J_l} 
\sum_{i=1}^I b_i g_{i,j}(\eta) 
= \frac{1}{y_l}\sum_{j \in J_l} 
\eta_j \sum_{i=1}^I A_{i,j}  \frac{b_i}{(A \eta)_i}, 
\quad
\text{and}
\\
x_j 
&= \frac{1}{\lambda_l}  \eta_j  \sum_{i=1}^I A_{i,j} \frac{b_i}{(A \eta)_i}, \quad j \in J_l.                                                       
\end{align*}                                             
ii)
By definition of $g_{i,j}$ it holds for any $\eta \in \R^J_{>0}$ that
$\sum_{j=1}^J g_{i,j} (\eta) = 1$.
Consequently, we have
$$G(x,x) = - \sum_{i=1}^I b_i \sum_{j=1}^J g_{i,j}(x) \log((Ax)_i) = F(x).$$
     Since the logarithm is concave, we obtain
    \[
    \log((Ax)_i) = \log\left(\sum_{j=1}^J g_{i,j}(\eta) \frac{A_{i,j}x_j}{g_{i,j}(\eta)}\right) \geq \sum_{j=1}^J g_{i,j}(\eta) \log\left(\frac{A_{i,j}x_j}{g_{i,j}(\eta)}\right) 
    \]
    and by $b \in \R^I_{>0}$ 
		and definition of $G$ finally
    $G(x,\eta) \geq F(x)$.
\end{proof}

By the previous result, 
we see that \eqref{alg:c_EMML} simply implements the iterative update
\begin{equation} \label{eq:voila}
x^{(r+1)} = \argmin_{x \in R} G(x, x^{(r)}),
\end{equation}
for an arbitrary starting vector $x^{(0)} \in \R_{>0}^J$.
Due to $(A x^{(r)})_i \leq N$ we obtain
\begin{align}\label{eq:estimate_a}
\lambda^{(r)}_l 
&= 
\frac{1}{y_l} \sum_{j \in J_l} x^{(r)}_j \sum_{i=1}^I A_{i,j} \frac{b_i}{(Ax^{(r)})_i}
\geq 
\underbrace{\frac{1}{y_l} \sum_{j \in J_l} 
x^{(r)}_j}_{=1}
\biggl(\min_{i \in [I]} b_i \biggr) 
\frac{\min_{j \in [J]} \bigl(\sum_{i=1}^N A_{i,j}\bigr)}{N} 
\nonumber
\\
&\geq 
\biggl(\min_{i \in [I]} b_i \biggr) 
\frac{\min_{j \in [J]} \bigl(\sum_{i=1}^N A_{i,j}\bigr)}{N} \eqqcolon \lambda_{\min} 
> 0, \quad r \in \N.
\end{align}

\begin{lemma}\label{lem:1}
    The sequence $(F(x^{(r)}))_r$ is monotonously decreasing and
    \begin{equation*}
    \|x^{(r+1)} - x^{(r)}\|_2 \to 0 \quad \text{as } r \to \infty.
    \end{equation*}
\end{lemma}

\begin{proof}
Since $G$ is a surrogate of $F$, we obtain with \eqref{eq:voila} that
\[
F(x^{(r)}) = G( x^{(r)},x^{(r)}) \ge G( x^{(r+1)},x^{(r)}) \ge F(x^{(r+1)}),
\]
so that the sequence of numbers $\left( F(x^{(r)}) \right)_r$ decreases monotonously.
Since it is also bounded from below, it converges to some number $F^*$.
We estimate
\begin{align*}
F(x^{(r)}) - F(x^{(r+1)}) 
&\ge
G( x^{(r)},x^{(r)}) - G( x^{(r+1)},x^{(r)})\\
&=
\sum_{i=1}^I b_i \sum_{j=1}^J g_{i,j}(x^{(r)}) 
\left( \log \frac{g_{i,j}(x^{(r)})}{A_{i,j} x_j^{(r)}} - \log \frac{g_{i,j}(x^{(r)})}{A_{i,j} x_j^{(r+1)}} \right)\\
&=
\sum_{i=1}^I b_i \sum_{j=1}^J \frac{A_{i,j} x_j^{(r)}}{(A x^{(r)})_j} \log \left( \frac{x^{(r+1)}_j}{x^{(r)}_j} \right)\\
&=
\sum_{l=1}^L \sum_{j \in J_l} x^{(r)}_j u^{(r)}_j \log \left( \frac{x^{(r+1)}_j}{x^{(r)}_j} \right)\\
&=
\sum_{l=1}^L \lambda_l^{(r)} \sum_{j \in J_l} x^{(r+1)}_j \log \left( \frac{x^{(r+1)}_j}{x^{(r)}_j} \right)\\
&=
\min_{l \in [L]} \lambda_l^{(r)} \KL (x^{(r+1)},x^{(r)}) 
\ge 
\lambda_{\min}
\KL (x^{(r+1)},x^{(r)}), 
\end{align*}
where we applied \eqref{eq:estimate_a} in the final estimate. Moreover, we can estimate $\KL (x^{(r+1)},x^{(r)})$ using
Taylor's expansion of $\KL (\cdot,x^{(r)})$ with some 
$\eta^{(r)} \in \{(1-t) x^{(r)} + t x^{(r+1)} : t     \in [0,1]\}$ 
as
    \begin{align*}
    \KL( x^{(r+1)},x^{(r)} ) 
		&= 
		\KL(x^{(r)},x^{(r)}) 
		+ 
		\nabla_x \KL(x^{(r)},x^{(r)})^\tT (x^{(r+1)} - x^{(r)}) 
		\\
    & \quad + 
		\frac12 (x^{(r+1)} - x^{(r)})^\tT 
		\nabla_x^2 \KL(\eta^{(r)},x^{(r)}) (x^{(r+1)} - x^{(r)})\\
		&= 
		\frac12 (x^{(r+1)} - x^{(r)})^\tT
		\text{diag} \left(\frac{1}{\eta^{(r)}}\right) 
		(x^{(r+1)} - x^{(r)})
		\\
    &\geq \|x^{(r+1)} - x^{(r)}\|_2^2.
    \end{align*}
In summary, we get 
\begin{equation*}
\|x^{(r+1)} - x^{(r)}\|_2^2 
\le 
\frac{1}{\lambda_{\min}}
(F(x^{(r)}) -	F(x^{(r+1)})) \to 0 \quad \text{as } r \to \infty.
\end{equation*}
\end{proof}

\begin{lemma}\label{cor:1}
    If $(x^{(r_k)})_k$ is a converging subsequence of $(x^{(r)})_r$ with limit $x^*$, then $(x^{(r_k+1)})_k$ also converges to $x^*$.
    Furthermore, any subsequential limit point $x^*$ of $(x^{(r)})_r$  satisfies
    \[
    x^*_j = \frac{1}{\lambda^*_l} x^*_j u^*_j, \quad j \in J_l, l \in [L],
    \]
    where 
    $\lambda^*_l 
    = 
    \frac{1}{y_l} \sum_{j \in J_l} x_j^* u_j^* > 0$, 
    $l \in [L]$ 
    and $u^* = A^{\tT}\bigl(\frac{b}{Ax^*}\bigr)$.
\end{lemma}

\begin{proof}
    The first statement immediately follows by \eqref{lem:1} 
    since every bounded sequence has a convergent subsequence. 
    Moreover, 
    let $(x^{(r_k)})_k$ be a converging subsequence with limit $x^*$, 
    then by contruction we obtain
    \[
    u^{(r_k)} \to u^* \quad \text{and} \quad \lambda^{(r_k)} \to \lambda^*, \quad k \to \infty.
    \]
    Due to \eqref{eq:estimate_a} we have $\lambda^*_l > 0$ for all $l \in [L]$. 
    Thus, by \eqref{lem:1} and definition of the iterates,
    \[
    x^*_j 
    = 
    \lim_{k \to \infty} x^{(r_k + 1)} 
    = 
    \lim_{k \to \infty} \frac{1}{\lambda^{(r_k)}_l} x^{(r_k)}_j u^{(r_k)}_j
    = 
    \frac{1}{\lambda^*_l} x^*_j u^*_j, \quad j \in J_l, l \in [L],
    \]
    as desired.
\end{proof}

In the following, we denote the componentwise multiplication of two vectors by $\odot$.
Furthermore, we extend this notation to $\lambda \in \R^L$ and $x \in \R^J$,
by defining
\[
(\lambda \odot x)_j 
\coloneqq 
\lambda_l x_j, \quad j \in J_l, l \in [L].
\]
In a similar fashion, 
for all $\pi \in \R^{I \times J}$, we define $\lambda \odot \pi \in \R^{I \times J}$ by 
\[
(\lambda \odot \pi)_{i,j} \coloneqq \lambda_l \pi_{i,j}, \quad i \in [I], j \in J_l, l \in [L].
\]

\begin{lemma}\label{lem:2}
    Let $(\lambda^*,x^*)$ be a subsequential limit point of $(\lambda^{(r)},x^{(r)})_r$, then
    \begin{equation}\label{est:3}
    \KL(\lambda^* \odot x^*, \lambda^{(r)} \odot x^{(r+1)}) \leq \KL(\lambda^* \odot x^*,\lambda^* \odot x^{(r)}), \quad r \in \N.
    \end{equation}
\end{lemma}

\begin{proof}
    For $r \in \N$, we set
    \begin{align*}
    q^{(r)}_{i,j} &\coloneqq A_{i,j} x_j^{(r)},\\
    \pi^{(r)}_{i,j} &\coloneqq \frac{1}{\lambda_l^{(r)}} x^{(r)}_j A_{i,j} \frac{b_i}{(Ax^{(r)})_i},\\
    \pi^*_{i,j} &\coloneqq \frac{1}{\lambda_l^*} x^*_j A_{i,j} \frac{b_i}{(Ax^*)_i}, \quad i \in [I], j \in J_l, l \in [L].
    \end{align*}
    The proof is divided into three steps, the first two steps consist of showing the inequalities
    \begin{align}
        \KL(\lambda^* \odot x^*, \lambda^{(r)} \odot x^{(r+1)}) &\leq \KL(\lambda^* \odot \pi^*, \lambda^{(r)} \odot \pi^{(r)}),  \label{est:1}\\
        \KL(\lambda^* \odot \pi^*, \lambda^{(r)} \odot \pi^{(r)}) + \KL(\lambda^{(r)} \odot \pi^{(r)},q^{(r)}) &\leq \KL(\lambda^* \odot \pi^*,q^{(r)}), \label{est:2}
    \end{align}
    which are then used to prove \eqref{est:3} in the third step. 

    We begin by showing \eqref{est:1}.
    For any $c \in \R^{I \times J}$, we define $\widetilde{c} \in \R^{I \times J}$ by
    \smash{$
    \widetilde{c}_{i,j} \coloneqq \frac{c_{i,j}}{\sum_{i' \in [I]} c_{i',j}}
    $}.
    By construction, we have $\sum_{i \in [I]} \pi^{(r)}_{i,j} = \frac{1}{\lambda_l^{(r)}} x^{(r)}_j u^{(r)}_j = x^{(r+1)}_j$, $j \in J_l$, $l \in [L]$. Similarly, we can leverage \eqref{cor:1} to obtain $\sum_{i \in [I]} \pi^*_{i,j} = x^*_j$, $j \in [J]$. Thus,
    \[
    \pi^*_{i,j} = \widetilde{\pi^*_{i,j}} x^*_j, \quad \text{and} \quad \pi^{(r)}_{i,j} = \widetilde{\pi^{(r)}_{i,j}} x^{(r+1)}_j, \quad i \in [I], j \in [J].
    \]
    It holds
    \begin{align*}
        &\KL(\lambda^* \odot \pi^*,\lambda^{(r)} \odot \pi^{(r)}) 
        = \sum_{\substack{i \in [I] \\ j \in J_l, l \in [L]}} \lambda^*_l \pi^*_{i,j} \log\left(\frac{\lambda^*_l  \pi^*_{i,j}}{\lambda^{(r)}_l \pi^{(r)}_{i,j}}\right)\\
        &= \sum_{\substack{i \in [I] \\ j \in J_l, l \in [L]}} \lambda^*_l \widetilde{\pi^*_{i,j}} x^*_j\log\left(\frac{\lambda^*_l \widetilde{\pi^*_{i,j}} x^*_j}{\lambda^{(r)}_l \widetilde{\pi^{(r)}_{i,j}} x^{(r+1)}_j}\right)\\
        &= \sum_{\substack{i \in [I] \\ j \in J_l, l \in [L]}} \lambda^*_l \widetilde{\pi^*_{i,j}} x^*_j\log\left(\frac{\lambda^*_l x^*_j}{\lambda^{(r)}_l x^{(r+1)}_j}\right) 
        + \sum_{\substack{i \in [I] \\ j \in J_l, l \in [L]}} \lambda^*_l \widetilde{\pi^*_{i,j}} x^*_j\log\left(\frac{\widetilde{\pi^*_{i,j}}}{\widetilde{\pi^{(r)}_{i,j}}}\right)\\
        &= \sum_{j \in J_l, l \in [L]} \lambda^*_l x^*_j\log\left(\frac{\lambda^*_l x^*_j}{\lambda^{(r)}_l x^{(r+1)}_j}\right) 
        + \underbrace{\sum_{\substack{i \in [I] \\ j \in J_l, l \in [L]}} \lambda^*_l \widetilde{\pi^*_{i,j}} x^*_j\log\left(\frac{\lambda^*_l  x^*_j \widetilde{\pi^*_{i,j}}}{\lambda^*_l x^*_j\widetilde{\pi^{(r)}_{i,j}}}\right)}_{\geq 0}\\
        &\geq \KL(\lambda^* \odot x^*, \lambda^{(r)} \odot x^{(r+1)}).
    \end{align*}
    Notably, without loss of generality the summation over $j$ in each line can be restricted to \smash{$x_j^* \neq 0$} and due to \eqref{prop:surrogate} it holds $x^{(r+1)}_j \neq 0$ so that all terms are well defined. We remark that the last estimate is obtained by \eqref{eq:Gibbs} together with
    \[
    \left.
\begin{array}{ll}
\sum_{\substack{i \in [I] \\ j \in J_l, l \in [L]}} \lambda^*_l \widetilde{\pi^*_{i,j}} x_j^* \\
\sum_{\substack{i \in [I] \\ j \in J_l, l \in [L]}} \lambda^*_l \widetilde{\pi^{(r)}_{i,j}} x_j^* 
\end{array}
\right\}
     = \sum_{j \in J_l, l \in [L]} \lambda^*_l x_j^* = \sum_{j \in [J]} x_j^* \sum_{i \in [I]} A_{i,j} \frac{b_i}{(Ax^*)_i} = \sum_{i\in [I]} b_i = 1.
    \]
    where we used \eqref{cor:1}.
    Hence we obtain the first estimate \eqref{est:1}.

    We proceed by showing \eqref{est:2}. For this, we define the following convex set
    \[
    \Delta^b \coloneqq \{c \in \R^{I \times J} : \sum_{j=1}^J c_{i,j} = b_i\}.
    \]
    For $c \in \Delta^b$, we have
    \begin{align}
        \KL(c,q^{(r)}) 
        &= \sum_{i,j} c_{i,j} \log\left(\frac{c_{i,j}}{q^{(r)}_{i,j}}\right)
        \nonumber
        \\
        &= \sum_{i,j} c_{i,j}\log\left(\frac{c_{i,j} b_i (Ax^{(r)})_i}{q^{(r)}_{i,j} b_i(Ax^{(r)})_i}\right)
        \nonumber
        \\
        &= 
        \sum_{i,j} c_{i,j}\log\left(\frac{c_{i,j} (Ax^{(r)})_i}{q^{(r)}_{i,j}b_i }\right)
        +
        \sum_{i,j} c_{i,j}\log\left(\frac{b_i}{(Ax^{(r)})_i}\right)
        \nonumber
        \\
        &= 
        \underbrace{\sum_{i,j} c_{i,j}\log\left(\frac{c_{i,j} (Ax^{(r)})_i}{q^{(r)}_{i,j}b_i }\right)}_{\geq 0}
        +
        \sum_{i} b_i \log\left(\frac{b_i}{(Ax^{(r)})_i}\right)
        \nonumber
        \\
        &\geq \KL(b,Ax^{(r)}), \label{eq:estimate_KL_c_q}
    \end{align}
    where the last estimate is again obtained by {\color{red} \eqref{eq:Gibbs}} together with
    \[
    \sum_{i\in [i], j \in [J]} c_{i,j} = \sum_{i \in [I]} b_i = 1 \quad \text{and} \quad \sum_{i\in [i], j \in [J]} \frac{q^{(r)}_{i,j} b_i}{(Ax^{(r)})_i} = \sum_{i \in [I]} b_i = 1.
    \]
    We have equality in \eqref{eq:estimate_KL_c_q} if and only if 
    \[
    c_{i,j} = q_{i,j}^{(r)} \frac{b_i}{(Ax^{(r)})_i} = x_j^{(r)} A_{i,j} \frac{b_i}{(Ax^{(r)})_i} = \lambda^{(r)}_l \pi^{(r)}_{i,j}.
    \]
    Now, set $f(t) \coloneqq \KL(t (\lambda^{(r)} \odot \pi^{(r)}) + (1-t)(\lambda^* \odot \pi^*),q^{(r)})$, $t \in [0,1]$. The derivative of $f$ is given by
    \[
    f'(t) = \sum_{i,j} \left(\lambda^{(r)}_l \pi^{(r)}_{i,j}  - \lambda^*_l \pi^*_{i,j}\right) 
    \left(\log\left(\frac{t \lambda^{(r)}_l \pi^{(r)}_{i,j} + (1-t)\lambda^*_l \pi^*_{i,j}}{q^{(r)}_{i,j}}\right) + 1 \right).
    \]
    By construction $\lambda^{(r)} \odot \pi^{(r)}, \lambda^* \odot \pi^*$ and any of their convex combinations are elements of $\Delta^b$. Hence
    \begin{align*}
        0 \geq f'(1) &= \sum_{i,j} \left(\lambda^{(r)}_l \pi^{(r)}_{i,j}  - \lambda^*_l \pi^*_{i,j}\right) 
        \log\left(\frac{\lambda^{(r)}_l \pi^{(r)}_{i,j}}{q^{(r)}_{i,j}}\right) + \underbrace{\sum_{i,j} \left(\lambda^{(r)}_l \pi^{(r)}_{i,j}  - \lambda^*_l \pi^*_{i,j}\right)}_{=0}\\
        &= \KL(\lambda^{(r)} \odot \pi^{(r)}, q^{(r)}) - \sum_{i,j} \lambda_l \pi^*_{i,j} \log\left(\frac{\lambda^{(r)}_l \pi^{(r)}_{i,j}}{q^{(r)}_{i,j}}\right)\\
        &= \KL(\lambda^{(r)} \odot \pi^{(r)}, q^{(r)}) 
        - 
        \sum_{i,j} \lambda_l \pi^*_{i,j} \log\left(\frac{\lambda^{(r)}_l \pi^{(r)}_{i,j}}{q^{(r)}_{i,j}} \frac{\lambda^*_l \pi^*_{i,j}}{\lambda^*_l \pi^*_{i,j}}\right)\\
        &= \KL(\lambda^{(r)} \odot \pi^{(r)}, q^{(r)}) 
        - 
        \KL(\lambda^* \odot \pi^*,q^{(r)})
        +
        \KL(\lambda^* \odot \pi^*,\lambda^{(r)} \odot \pi^{(r)})
    \end{align*}
    which yields \eqref{est:2}.

    We finish the proof by showing \eqref{est:3}.
    It holds
    \begin{align*}
        &\KL(\lambda^* \odot x^*, \lambda^{(r)} \odot x^{(r+1)}) \\
        \stackrel{\eqref{est:1}}{\leq} 
        &\KL(\lambda^* \odot \pi^*, \lambda^{(r)} \odot \pi^{(r)})\\
        \stackrel{\eqref{est:2}}{\leq} 
        &\KL(\lambda^* \odot \pi^*,q^{(r)}) - \KL(\lambda^{(r)} \odot \pi^{(r)},q^{(r)}) \\
        = 
        &\sum_{\substack{i \in [I] \\ j \in J_l, l \in [L]}} \lambda_l^* \pi^*_{i,j} \log\left(\frac{x^*_j A_{i,j} \frac{b_i}{(Ax^*)_i}}{A_{i,j} x^{(r)}_j}\right) - \sum_{\substack{i \in [I] \\ j \in J_l, l \in [L]}} \lambda^{(r)}_l \pi^{(r)}_{i,j} \log\left(\frac{x^{(r)}_j A_{i,j} \frac{b_i}{(Ax^{(r)})_i}}{A_{i,j}x^{(r)}_j}\right)\\
        = 
        &\sum_{\substack{i \in [I] \\ j \in J_l, l \in [L]}} \lambda_l^* \pi^*_{i,j} \log\left(\frac{x^*_j}{x^{(r)}_j}\right) 
        +
        \!\!\!\!\!
        \sum_{\substack{i \in [I] \\ j \in J_l, l \in [L]}} \lambda_l^* \pi^*_{i,j} \log\left(\frac{b_i}{(Ax^*)_i}\right) 
        - 
        \!\!\!\!\!
        \sum_{\substack{i \in [I] \\ j \in J_l, l \in [L]}} \lambda^{(r)}_l \pi^{(r)}_{i,j} \log\left(\frac{b_i}{(Ax^{(r)})_i}\right)\\
        = 
        &\sum_{j \in J_l, l \in [L]} \lambda_l^* x_j^*\log\left(\frac{x^*_j}{x^{(r)}_j}\right) 
        +
        \sum_{i \in [I]} b_i \log\left(\frac{b_i}{(Ax^*)_i}\right) 
        - 
        \sum_{i \in [I]} b_i \log\left(\frac{b_i}{(Ax^{(r)})_i}\right)\\
        = 
        &\; \KL(\lambda^* \odot x^*, \lambda^* \odot x^{(r)})
        +
        \underbrace{\sum_{i,j} b_i \log\left(\frac{1}{(Ax^*)_i}\right) 
        - 
        \sum_{i} b_i \log\left(\frac{1}{(Ax^{(r)})_i}\right)}_{= F(x^*) - F(x^{(r)}) \leq 0}\\
        \leq 
        &\; \KL(\lambda^* \odot x^*, \lambda^* \odot x^{(r)}).
    \end{align*}
    Above we have used that $\sum_{i} \pi^*_{i,j} = x^*_j$ in the sixth line which follows by \eqref{cor:1} and the monotonicity of $(F(x^{(r)}))_r$ provided by \eqref{lem:1}.
\end{proof}

\begin{lemma}\label{lem:KL_theta_u_decreases}
    Let $(\lambda^*,x^*)$ be subsequential limit point of $(\lambda^{(r)},x^{(r)})_r$. Then it holds
    \begin{equation}
    \KL(\lambda^* \odot x^*,x^{(r)} \odot u^{(r)}) \leq \KL(\lambda^* \odot x^*,x^{(r-1)} \odot u^{(r-1)}), \quad r \in \N.
    \end{equation}
\end{lemma}

\begin{proof}
Let $r \in \N$, \eqref{lem:2} yields
\begin{align*}
    &\KL(\lambda^* \odot x^*,x^{(r)} \odot u^{(r)}) = \KL(\lambda^* \odot x^*,\lambda^{(r)} \odot x^{(r+1)})
    \leq \KL(\lambda^* \odot x^*, \lambda^* \odot x^{(r)})\\
    &= \sum_{j \in J_l, l \in [L]} \lambda^*_l x_j^* \log\left(\frac{\lambda^*_l x^*_j}{\lambda^*_lx_j^{(r)}}\right)\\
    &= \sum_{j \in J_l, l \in [L]} \lambda^*_l x_j^* \log\left(\frac{\lambda^{(r-1)}_l \lambda^*_lx^*_j}{\lambda^*_lx^{(r-1)}_j u^{(r-1)}_j}\right)\\
    &= \underbrace{\sum_{j \in J_l, l \in [L]} \lambda^*_l x_j^* \log\left(\frac{\lambda^*_l x^*_j}{x^{(r-1)}_j u^{(r-1)}_j}\right)}_{=\KL(\lambda^* \odot x^*, x^{(r-1)} \odot u^{(r-1)})} + \sum_{j \in J_l, l \in [L]} \lambda^*_l x_j^* \log\left(\frac{\lambda^{(r - 1)}_l}{\lambda^*_l}\right).
\end{align*}
Hence, showing $\sum_{j \in J_l, l \in [L]} \lambda^*_l x_j^* \log\left(\frac{\lambda^{(r - 1)}_l}{\lambda^*_l}\right) \leq 0$ finishes the proof.
For this term, we obtain
\begin{align*}
\sum_{j \in J_l, l \in [L]} \lambda^*_l x_j^* \log\left(\frac{\lambda^{(r-1)}_l}{\lambda^*_l}\right) &= \sum_{l \in [L]} \lambda^*_l \underbrace{\left(\sum_{j \in J_l} x_j^*\right)}_{=y_l} \log\left(\frac{\lambda^{(r-1)}_l}{\lambda^*_l}\right)\\
&= - \sum_{l \in [L]} y_l \lambda^*_l\log\left(\frac{y_l \lambda^*_l}{y_l \lambda^{(r-1)}_l}\right)\\
&\leq 0,
\end{align*}
where we used \eqref{eq:Gibbs} in the last step together with
\begin{align*}
&\sum_{l \in [L]} y_l \lambda^*_l = \sum_{l \in [L]} y_l \frac{1}{y_l} \sum_{j \in J_l} x^*_j u^*_j = \sum_{j \in [J]} x^*_j \sum_{i \in [I]} A_{i,j} \frac{b_i}{(Ax^*)_i} = \sum_{i \in [I]} b_i = 1
\end{align*}
and $\sum_{l} y_l \lambda^{(r-1)}_l = 1$ which can be shown analogously.
\end{proof}

Finally, we can give the proof of \eqref{thm:convergence}.

\begin{proof}[Proof of Theorem \ref{thm:convergence}]
The monotonic decrease of the objective is already shown in \eqref{lem:1}.
We show that $(x^{(r)})_r$ converges. 
Since $(x^{(r)})_r$ is contained in the compact set of probability vectors, 
we can pick converging subsequences $(x^{(r_k)})_k$ and $(x^{(r_m')})_m$ 
with limits $x^*$ and $\tilde{x}$, respectively. 
We show $x^* = \tilde{x}$. 
We denote
\[
u^*_j 
= 
A^{\tT} \left(\frac{b}{Ax^*}\right), 
\quad 
\lambda^*_l 
= 
\frac{1}{y_l} \sum_{j \in J_l} x^*_ju^*_j, 
\quad 
\tilde{u}_j 
= 
A^{\tT} \left(\frac{b}{A\tilde{x}}\right), 
\quad 
\tilde{\lambda}_l 
= 
\frac{1}{y_l} \sum_{j \in J_l} \tilde{x}_j \tilde{u}_j.
\]
Clearly $\lambda^*$, $u^*$, $\tilde{\lambda}$ and $\tilde{u}$ 
are the limits of $(\lambda^{(r_k)})_k$, $(u^{(r_k)})_k$, $(\lambda^{(r_m')})_m$ 
and $(u^{(r_m')})_m$, respectively.
Let $k,m \in \N$ so that $r_m' \geq r_k$. 
Due to \eqref{lem:KL_theta_u_decreases},
we obtain
\[
\KL(\lambda^* \odot x^*, \lambda^{(r_m')} \odot x^{(r_m')}) 
\leq 
\KL(\lambda^* \odot x^*, \lambda^{(r_k)} \odot x^{(r_k)}).
\]
Taking the limit $m \to \infty$ yields
\[
\KL(\lambda^* \odot x^*, \tilde{\lambda} \odot \tilde{x}) 
\leq 
\KL(\lambda^* \odot x^*, \lambda^{(r_k)} \odot x^{(r_k)}) 
\to 0, 
\quad k \to \infty.
\]
We arrive at 
$\lambda^*_l x^*_j 
= 
\tilde{\lambda}_l \tilde{x}_j$ 
for all $j \in J_l$, $l \in [L]$. 
By construction, 
$\sum_{j \in J_l} x^*_j 
= 
\sum_{j \in J_l} \tilde{x}_j = y_l > 0$, 
$l \in [L]$, 
we obtain $\lambda^* = \tilde \lambda$ which have non-zero entries
and thus also $x^* = \tilde{x}$. 
Since every converging subsequence of $(x^{(r)})_r$ 
converges to the same limit $x^*$, 
the entire sequence converges.

To finish the proof we show that $x^*$ is indeed a minimizer of $F + \iota_R$.
More precisely, 
we show that $x^*, \lambda^*$ fulfill the KKT conditions
\begin{align*}
\sum_{i \in I} A_{i,j} \frac{b_i}{(A x)_i} 
&= \lambda_l \quad \text{if } x_j >0, \; j \in J_l, l \in [L],
\\
\sum_{i \in I} A_{i,j} \frac{b_i}{(A x)_i} 
&\le \lambda_l \quad \text{if } x_j = 0, \; j \in J_l, l \in [L],
\\
\sum_{j \in J_l} x_j &= y_l, \quad l \in [L].
\end{align*}
By construction, 
$x^*$ readily fulfills the last condition. 
Moreover, \eqref{cor:1} gives
\[
x_j^* 
=
\frac{1}{\lambda^*_l} x^*_j 
\sum_{i\in [I]} A_{i,j} \frac{b_i}{(Ax^*)}, 
\quad j \in J_l, l \in [L].
\]
Hence, 
for any $j \in J_l$, $l \in [L]$ for which $x_j^* > 0$, 
we can divide both sides by $x_j^*$ 
to see that $x^*$ fulfills the first KKT condition with multiplier $\lambda^*$. 
Finally, to see the second KKT condition, 
assume that there exists some $j \in J_l$, $l \in [L]$ with
\[
u_j^* = \sum_{i \in I} A_{i,j} \frac{b_i}{(A x^*)_i} > \lambda_l^*.
\]
Since $u_j^*$ is the limit of $(u^{(r)}_j)_r$, 
there exists $\varepsilon > 0$ and $N \in \N$, 
so that $\frac{u^{(r)}_j}{\lambda^{(r)}_l} \geq 1 + \varepsilon$ 
for all $r > N$. 
Then, it holds
\[
x_j^* 
= \lim_{r \to \infty} x^{(r)}_j 
= \lim_{r \to \infty} x^{(r-1)}_j \frac{u^{(r-1)}_j}{\lambda^{(r-1)}_l} 
= \cdots
=\underbrace{x^{(N)}_j}_{\neq 0} \lim_{r \to \infty} \prod_{n = N}^r \underbrace{\frac{u^{(n)}_j}{\lambda^{(n)}_l}}_{ \geq 1 + \varepsilon} = \infty,
\]
which yields the desired contradiction and thus finishes the proof.
\end{proof}

\bibliographystyle{abbrv}
\bibliography{references}

\end{document}